\documentclass{article}

    \PassOptionsToPackage{numbers, compress}{natbib}

\usepackage[main, final]{neurips_2026}

\usepackage[utf8]{inputenc} 
\usepackage[T1]{fontenc}    
\usepackage{hyperref}       
\usepackage{url}            
\usepackage{booktabs}       
\usepackage{amsfonts}       
\usepackage{nicefrac}       
\usepackage{microtype}      
\usepackage{xcolor}         
\usepackage{amsmath}
\usepackage{graphicx}
\usepackage{wrapfig}
\usepackage{algorithm}
\usepackage{algorithmic}
 \usepackage{multirow}
\usepackage{setspace}
\usepackage{amsthm}
\usepackage{amssymb}
\usepackage{array}
\usepackage{booktabs}
\usepackage{cleveref}

\newtheorem{theorem}{Theorem}[section]
\newtheorem{proposition}[theorem]{Proposition}

\newtheorem{definition}[theorem]{Definition}

\title{TimeES: Probabilistic and Deterministic Time Series Forecasting via Evolutionary Spectra}

\author{%
  Weiwei Ye \quad Renhe Jiang$^{*}$ \quad Hangchen Liu \quad Dongyuan Li \quad Yoshihide Sekimoto \\
  The University of Tokyo\quad $^{*}$Corresponding author
}
\begin{document}

\maketitle
\begin{abstract}

Real-world time series are inherently non-stationary, with trends, periodic patterns, and uncertainty evolving over time. While the Fourier domain offers a natural lens to model time series, current deep learning approaches do not explicitly model evolution and randomness in the Fourier spectra, which limits their ability to accurately predict both the expected trajectory and its uncertainty in non-stationary time series.  Motivated by Evolutionary Spectra (ES) theory, we propose \textbf{TimeES}, a general framework that enables probabilistic and deterministic forecasting via the evolutionary spectra theory.  Specifically, we derive a parameterizable evolutionary spectra formulation, recasting non-stationary random process modeling as learning an evolving representation modulated by random variables. Furthermore, we reduce the complexity of the estimated spectra from $\mathcal{O}(NM)$ to $\mathcal{O}(NK)$, where $K \ll M/2$, by exploiting Hermitian symmetry and spectral energy sparsity for frequency selection. Based on a simple linear backbone, our proposed TimeES achieves consistent state-of-the-art performance across both deterministic and probabilistic forecasting tasks, with high efficiency and interpretability. Code is available at: \url{https://github.com/wwy155/TimeES}.
\end{abstract}

\section{Introduction}
\label{submission}
General time series forecasting involves both deterministic forecasting and uncertainty estimation, and plays a critical role in various domains, such as traffic~\cite{ermagun2018spatiotemporal,liu2023spatio} and finance~\cite{duan2022factorvae}. The frequency domain offers a natural way to model trend and periodic patterns; however, under non-stationarity, these patterns often evolve rapidly and are entangled with uncertainty~\cite{liu2022non}, posing significant challenges.

Regarding deterministic forecasting, current frequency-domain approaches typically rely on analytically derived Fourier features (e.g., Fourier transform (FT)~\cite{cohen1995time} or short-time Fourier transform (STFT)), and then apply models to learn mappings from these features to predictions~\cite{xu2023fits, yi2024frequency}. However, because the Fourier features are obtained analytically, these methods remain bound by the intrinsic assumptions of a locally invariant spectrum~\cite{priestley1988spectral}. Consequently, they inherently fail to capture the spectra evolution that characterizes non-stationary time series \cite{ye2024frequency}.

In parallel, many existing probabilistic forecasting methods inherit implicit stationarity or i.i.d. assumptions from deep generative frameworks originally designed for computer vision~\cite{su2025diffusion}, while few works explore the fundamental theoretical framework specifically designed for modeling non-stationary random processes in a natural and principled manner. In real-world time series, data often exhibit structured deterministic dynamics, such as fixed periodic patterns~\cite{fan2024deep}, that are highly predictable and should not be absorbed into aleatoric uncertainty~\cite{li2025diffusion}. By conflating these systematic components with noise, such methods may misattribute the spectral component to random fluctuations, potentially failing to distinguish deterministic patterns from stochastic variations~\cite{chen2024probabilistic}.

To illustrate this, we present a case study on the ETTm1 dataset. As shown in Figure~\ref{fig:introduction}, the signal exhibits time-varying frequency components and substantial stochastic fluctuations. However, the FFT computed over the entire sequence fails to capture time-varying spectral contents; besides, the STFT  suffers from spectral leakage  and limited frequency resolution, resulting in 
 spikes between the trend and daily spectrum and a blurred spectrogram (as seen in Figure~\ref{fig:introduction} (b)); for probabilistic forecasting, even recent state-of-the-art methods such as TMDM~\cite{li2024tmdm} and NsDiff~\cite{ye2025non} erroneously treat periodic components as stochastic part, leading to poor performance.

 \begin{wrapfigure}{r}{0.5\linewidth}
 \centering
\vspace{-2pt}
\includegraphics[width=1\linewidth]{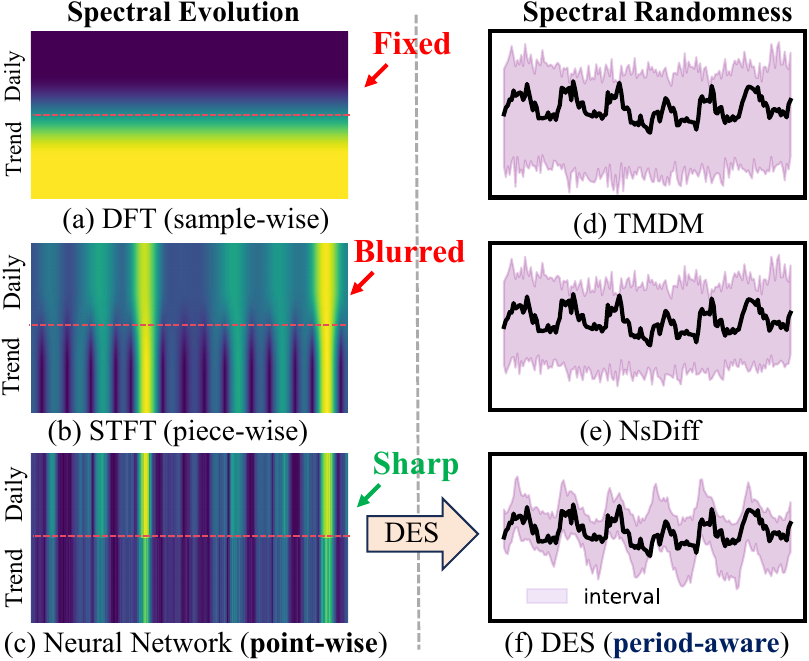}
\caption{A case study on the ETTm1 using 7 days samples with daily and trend spectrum~(0 and 1 frequency). (a-c): spectra obtained by different methods; (d-f): probabilistic forecasting results of different methods. TimeES produces point-wise spectra via neural network and quantify period-aware uncertainty interval.}
  \label{fig:introduction}
\end{wrapfigure}

To address these limitations, we turn to evolutionary spectra~(ES) theory~\cite{priestley1988spectral}, a cornerstone of non-stationary stochastic time series analysis that theoretically defines a time-varying spectrum but has remained underexploited in practice due to the lack of reliable methods for estimating the evolutionary spectra~\cite{benowitz2015determining}. In particular, we derive a parameterized formulation of the evolutionary spectra, which can be interpreted as both a decompositable and generative model and thus the evolutionary spectra can be effectively estimated as a representation learning problem. Building upon the formulation, we propose TimeES, a foundational forecasting framework that offers both physical interpretable decomposition and  probabilistic uncertainty quantification capabilities. As shown in Figure~\ref{fig:introduction}, TimeES yields a sharp and temporally resolved spectral estimate~(Figure ~\ref{fig:introduction}~(c)) and simultaneously captures the intrinsic randomness embedded in periodic patterns~(Figure ~\ref{fig:introduction}~(f)). 

In summary, our contributions are: (1) Motivated by the evolutionary spectra theory, we formulate \textbf{D}iscrete \textbf{E}volutionary \textbf{S}pectra~(\textbf{DES}), a parameterized point-wise generative model for time series,  in which the time-varying spectral density is explicitly represented and thus can be estimated directly from data.  (2) We propose \textbf{TimeES}, a framework that estimates and predicts evolutionary spectra using DES which supports both probabilistic generation and interpretable decomposition, thereby enabling both deterministic and probabilistic forecasting. We reduce the complexity of the estimation of the evolutionary spectra from $\mathcal{O}(NM)$ to $\mathcal{O}(NK)$, where $K \ll M/2$ (Section~\ref{Sec:4.2}) by exploiting Hermitian symmetry and spectral energy sparsity. (3) Through extensive experiments, we demonstrate that TimeES achieves consistent state-of-the-art performance in probabilistic and deterministic forecasting tasks. Notably, on the most challenging probabilistic forecasting benchmarks, TimeES outperforms recent baselines \textbf{17.23\%} on average with high efficiency, highlighting its superior ability to forecast time series.

\section{Related work}

\subsection{Evolutionary Spectra for Non-stationary Process}
To formally account for non-stationarity, Priestley proposed evolutionary spectra (ES) theory~\cite{priestley1965evolutionary,priestley1988spectral}. However, the ES is fundamentally non-identifiable due to infinite many valid solutions. There exist some techniques to estimate the ES, though these depend on very limiting conditions such as local stationarity or narrow-bandness~\cite{schillinger2010accurate, spanos2004evolutionary, spanos2005stochastic, von1996wavelet}. Recent approaches rely on perturbation-based optimization methods~\cite{benowitz2015determining}, requiring a strict initialization. Our proposed method estimates the ES directly from the data, by formulating the ES estimation as a representation learning problem.
\subsection{Deterministic Time Series Forecasting}
Recent years have witnessed significant advances in deep learning for deterministic time series forecasting, architectures such as RNNs~\cite{du2021adarnn} and Transformers~\cite{zhou2021informer} have demonstrated strong performance by capturing temporal dependencies through sequential or attention-based mechanisms~\cite{jiang2023spatio, liu2023itransformer, wu2021autoformer, wu2022timesnet}.  To capture the inherent periodic patterns, recent frequency-aware models~\cite{xu2023fits, yi2024frequency, zhou2022fedformer} explicitly incorporate spectral components to improve forecasting performance and efficiency. To further address non-stationarity, recent works explicitly model distributional or frequency shifts~\cite{kim2021reversible, ye2024frequency}. To model complex dynamics beyond deterministic assumptions, Koopman operator theory~\cite{liu2024koopa, wang2022koopman} and Dynamic Mode Decomposition (DMD)~\cite{schmid2022dynamic} offer powerful alternatives by providing linear representations of nonlinear dynamical systems in lifted feature spaces. Koopman models also learn continuous spectra~\cite{lusch2018deep} and admit probabilistic formulations~\cite{mallen2021deep, naiman2024generative, zheng2025koonpro}, but their spectra decompose an evolution operator in a lifted space; evolutionary spectra~\cite{priestley1965evolutionary} instead model the time-varying power spectra of the observed non-stationary random process itself, making them well-suited for data with non-stationary stochasticity (Appendix~\ref{apdx:koopman_dmd}).


\subsection{Probabilistic Time Series Forecasting}
Early efforts built on GANs~\cite{mogren2016c, yoon2019time} or VAEs~\cite{desai2021timevae} to capture the temporal structure, but their training instability has motivated a shift toward other alternative generative paradigms. Specifically, diffusion-based models, have recently been adapted to time series due to their strong performance in modeling complex distributions~\cite{chen2023provably, rasul2021autoregressive, shen2023non, shi2023diffusion, tashiro2021csdi, yuan2024diffusion}. However, most of these methods assume stationary distributions. To address this, NsDiff~\cite{ye2025non} allows a varying endpoint distribution, thereby better capture non-stationarity.  Nevertheless, nearly all existing approaches rely on methods originally developed for computer vision, which assume that data are generated i.i.d. In contrast, TimeES is explicitly grounded in the premise that data arise from a non-stationary random process.

We discussed related works concerning major methodology; refer to Appendix~\ref{apdx:full_related_works} for more details.

\section{Preliminary}

Throughout this paper, we use uppercase letters (e.g., X) to denote random variables, lowercase letters for scalars, and lowercase boldface letters for vectors. Specifically, we denote an input sample as $\mathbf{x}\in\mathbb{R}^N$  and the corresponding output sample as $\mathbf{y}\in\mathbb{R}^O$. We begin by considering a continuous-parameter stochastic process $\{X_t : t \in \mathbb{R}\}$, and review key preliminaries from non-stationary stochastic process and   the evolutionary spectra theory.



In the stationary case, consider a continuous-time stochastic process $\{X_t\}$ (e.g., white noise). Since realizations of such a process are typically not absolutely integrable over $(-\infty, \infty)$, the classical Fourier integral does not exist. The following theorem provides a general extension by introducing a random spectral measure defined as the followings:

\begin{theorem}[Spectrum Representation Theorem] The spectrum representation of a zero-mean, continuous, stationary process $\{X_t\}$ is:
    \begin{equation}
            X_t=\int_{-\infty}^{\infty} e^{i \omega t} d Z(\omega), 
    \label{eq:spectrum_representation_theorem}
    \end{equation}
    \label{theorem:spectral_representation}
\end{theorem}
where \( Z(\omega) \) is a \textit{complex random measure} satisfying \( dZ(\omega) = \mathcal{O}(\sqrt{d\omega}) \). \textit{Particularly, $X_t$ admits a white noise process $\mathcal{N}_t(\mathbf{0}, \mathbf{I})$ if we sample $dZ(\omega)$ from a standard complex Gaussian $\mathcal{CN}(\mathbf{0}, \mathbf{I})$.}  The transform expressed is called the Fourier–Stieltjes transform, as opposed to the classical Fourier transform. To facilitate  subsequent results, we provide a brief proof and illustration of Theorem~\ref{eq:spectrum_representation_theorem} in Appendix~\ref{apdx:proof:spectral_representation} (see \citet{priestley1988spectral}, Sections~4.7–4.11, for detailed mathematical treatment). 





In real-world settings, data-generating processes are often non-stationary, as a process is deemed non-stationary when its statistical and spectral properties vary over time. To address this, \citet{priestley1988spectral} proposed evolutionary spectral theory, which extends Equation~\eqref{eq:spectrum_representation_theorem} to the non-stationary regime. Specifically, for a non-stationary process with time-varying distributions and frequencies, it is tempting to define a varying quantity as proposed as:


 \begin{definition}[Evolutionary Spectra] Given a non-stationary random process $\{X_t\}$, assuming a deterministic square-integrable function $A(t, \omega) e^{i \omega t}$ and a random measure $dZ(\omega)$, $X_t$ is defined as:
    \begin{equation}
    X_t=\int_{-\infty}^{\infty} A(t, \omega) e^{i \omega t} d Z(\omega),
    \label{eq:evolutionary_spectra1}
    \end{equation}
    \label{theorem:evolutionary_spectra}
\end{definition}
where $A(t, \omega)$ is the complex amplitude and the energy spectral density is defined as $dS(t, \omega) = |A(t, \omega)|^2$. The representation in Definition~\ref{theorem:evolutionary_spectra} provides a principled foundation for modeling non-stationary random process through time-varying spectral characteristics. Identifying evolutionary spectra has significant practical importance~\cite{grigoriu1993spectral}, and is critical for effective forecasting, as they reflect the non-stationary random structures across the Fourier spectra with random consideration. However, in most scenarios, the true evolutionary spectra are unknown and admit infinite amounts of solutions, which poses a major challenge~\cite{spanos2004evolutionary}. In this paper, we propose a tractable formulation of this equation so the ES can be estimated through representation learning.

\section{Methodology}
In this section, we present the proposed framework, TimeES. We first discuss the main theoretical result underlying the proposed generative model, DES, which reformulates evolutionary spectra estimation as a tractable representation learning problem. Next, we describe how this model effectively estimates the evolutionary spectra and leverages it for general time series forecasting tasks. Then, we discuss how TimeES can degenerate to the DFT and STFT. An overview of the TimeES framework is provided in Figure~\ref{fig:overview}.

\begin{figure*}[!h]
\begin{center}
\centerline{\includegraphics[width=1\linewidth]{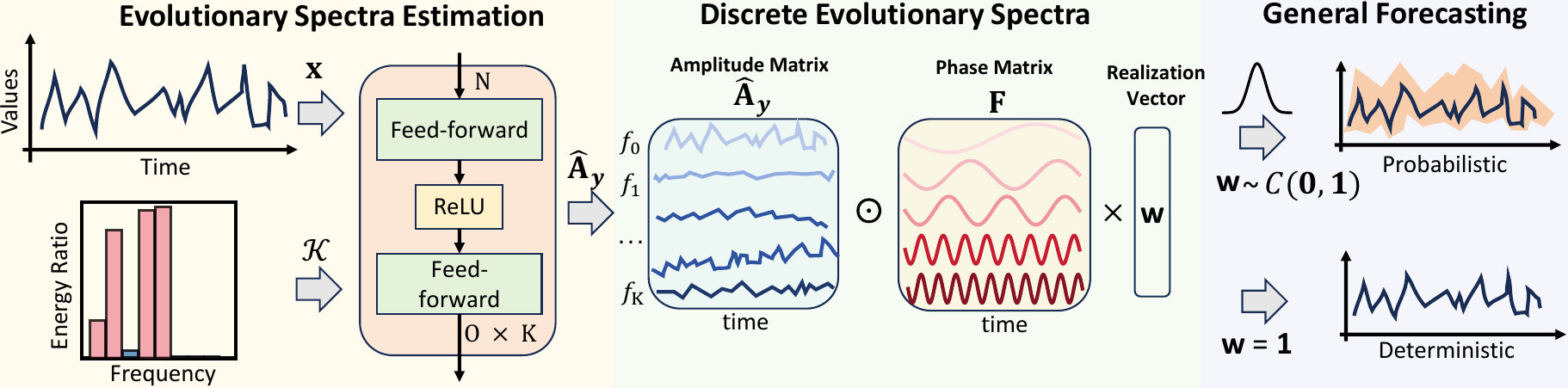}}
\caption{The overview of TimeES. TimeES estimates the interpretable amplitude matrix $\mathbf{A}$ by a neural network; then the DES is used to synthesize time series with a physical phase matrix $\mathbf{F}$ and a probabilistic realization vector $\mathbf{w}$, enabling general time series forecasting.}
\label{fig:overview}
\end{center}
\vskip -0.3in 
\end{figure*}

\subsection{Discrete Evolutionary Spectra}
Definition~\ref{theorem:evolutionary_spectra} establishes a direct relationship between the instantaneous amplitude function $A(t, \omega)$ and the continuous-time process $X_t$ at each time $t$. Motivated by this connection, we formulate Definition~\ref{theorem:evolutionary_spectra} as a parameterized generative model, which is the main theorem of this paper:


\begin{theorem}[Discrete Evolutionary Spectra]
Given discretely sampled data $\{X_n\}_{n=1}^N$, define a set of discrete frequencies 
$
\omega_k = \frac{2\pi k}{M}, \quad k = 0, 1, \ldots, M-1.
$
Then $X_n$ admits:
\begin{equation}
\label{eq:discrete_evolutionary_rep}
X_n = \frac{1}{\sqrt{M}} \sum_{k=0}^{M-1} A\!\left(n, \omega_k\right) W_k \, e^{i 2\pi k n / M},
\end{equation}
\label{theorem:main_theorem}
\end{theorem}
where $W_k \sim \mathcal{CN}(\mu_{W_k},\sigma_{W_k})$ are independent complex Gaussian random variables, and $A(n, \omega_k)$ denotes the evolutionary amplitude spectrum at time $n$ and frequency $\omega_k$.  We provide the proof in Appendix \ref{section:proof:maintheorem}. The parameter $M$ can be seen as the frequency length in FFT or STFT, we thus follow this convention in our experiments to use $M=N$ to cover the full discrete frequencies. Note that, this equation describes a single realization of the stochastic process $\{X_n\}$ corresponding to a particular draw of the complex random variables $W_k$. By sampling $W_k$ multiple times, we can perform generative modeling of the signal. \textit{In particular, in the special case $A=1$, $X_n$ is a random variable following a standard normal distribution as stated in Definition~\ref{theorem:evolutionary_spectra}. }
 
Importantly, unlike many prior approaches that adopt score-based generative models from computer vision which typically assume that data are drawn from instance-independent distributions, DES instead assumes that the observed time series is sampled from an oscillatory non-stationary process. This formulation is therefore far better suited to capturing the inherent non-stationary distribution in time series data. Presenting Theorem~\ref{theorem:main_theorem} in matrix form will be more convenient for later text:




\begin{proposition}
Theorem~\ref{theorem:main_theorem} can be equivalently expressed as
\begin{equation}
    \mathbf{x} =  \frac{1}{\sqrt{M}}\left( \, \mathbf{A} \odot \mathbf{F} \right) \mathbf{w},\label{eq:main_theorem}
\end{equation}
\label{proposition:matrix_form_of_main_theorem}
\end{proposition}

where $\mathbf{x} = [x_0, \dots, x_{N-1}]^\top \in \mathbb{C}^N$,  $\mathbf{w} = [W_0, \dots, W_{M-1}]^\top \in \mathbb{C}^M$,  $\mathbf{A}_{n,k} = A(n, \omega_k)$, $\mathbf{F}_{n,k} = e^{i 2\pi k n / M}$. In Equation~\eqref{eq:main_theorem}, $\odot$ denotes the Hadamard product.  To this end, we establish a generative model for non-stationary time series where $\mathbf{F}$ represents the physical frequency basis, $\mathbf{A}$ assigns evolutionary amplitude to each of these components; $\mathbf{w}$ encodes the probabilistic nature of the signal, so that any observed realization is a sample path of the underlying random process.

\subsection{Probabilistic and Deterministic Forecasting via Evolutionary Spectra}\label{Sec:4.2}
Currently, there is no analytic or numerical way to determine evolutionary spectra~\cite{benowitz2015determining}. Given that Theorem~\ref{theorem:main_theorem} establishes a point-wise and tractable relationship between the evolutionary spectrum and the sampled time series $X_n$, we consider a natural learning strategy: train a neural network to estimate $\mathbf{A}$ by matching the reconstructed signal with observed data. In principle, this approach requires a network parameterization with an output size of $\mathcal{O}(NM)$; and to capture all possible spectral components present in the data, one would set $M \geq N$ (e.g., $M = N$ for a full DFT basis), resulting in an output size of at least $\mathcal{O}(N^2)$. However, we show that by exploiting two inherent properties of real-world time series (Hermitian symmetry and spectral sparsity) the computational complexity can be reduced to $\mathcal{O}(NK)$, where $K \ll M/2$.

 







\textbf{Hermitian Symmetry}. Since most real-world time series are inherently real-valued, this requirement imposes a structural constraint on the complex amplitude spectrum $A(n, \omega_k)$ in Theorem~\ref{theorem:main_theorem}. Specifically, \textit{Hermitian symmetry}~\cite{roberts1987digital} must hold across the frequency dimension to ensure that the inverse spectral synthesis yields a real-valued sequence, which can be defined as:
\begin{theorem}
Let the signal \( X_n \) be synthesized as Equation~\eqref{eq:discrete_evolutionary_rep} and define $ B(n, k) = A(n, \omega_k) W_k $. Then \( X_n \in \mathbb{R} \) for all \( n \) if and only if the following Hermitian symmetry condition holds for all \( n \) and all \( k = 0, 1, \dots, M-1 \):
\begin{equation}
\label{eq:hermitian}
B(n, \omega_k) = \overline{B(n, \omega_{(M - k) \bmod M})}.
\end{equation}
\label{theorem:hermitian}
\end{theorem}
\vskip -0.1in
See the proof in Appendix~\ref{section:proof:hermitian}. In particular: The DC component (\( k = 0 \)) must be real: \( A(n, \omega_0) \in \mathbb{R} \). If \( M \) is even, the Nyquist component (\( k = M/2 \)) must also be real. Consequently, the entire spectrum is uniquely determined by the amplitudes at the non-redundant frequency indices $\left\{ 0, 1, \dots, \left\lfloor \frac{M}{2} \right\rfloor \right\}$, with total $M_h =\left\lfloor \frac{M}{2} \right\rfloor + 1$ frequencies. By exploiting Hermitian symmetry in Theorem~\ref{theorem:main_theorem}, we only need to estimate half of the frequency components.

\textbf{Spectral energy sparsity}. Moreover, in time series, not all frequency components contain useful information~\cite{donoho1998data}. While many prior methods exploit this property by selecting frequencies based on fixed thresholds~\cite{xu2023fits, zhou2022fedformer} or a pre-specified number of top components~\cite{wu2022timesnet, ye2024frequency}, the number of frequencies $K$ is often heuristic and lacks principled guidance. Motivated by the observation that real-world signals typically concentrate the majority of their energy $ \|\mathbf{\mathbf{x}}\|^2$ in only a few frequencies, we propose an \textit{energy-based frequency selection strategy}. Specifically, we first compute the average periodogram then sort all frequencies by their energy $I(k)$ in descending order and select the smallest subset of frequencies $\mathcal{K}$ such that the cumulative energy accounts for at least $r$ of the total:
\begin{equation}
I(k) = \mathbb{E}_{x \sim \mathcal{D}^{\text{train}}} \left[ |\hat{X}_k|^2 \right], \quad \frac{\sum_{k \in \mathcal{K}} I(k)}{\sum_{k=0}^{M-1} I(k)} \geq r,
\end{equation}
where $\hat{X}_k$ are the Fourier coefficients obtained by using Fourier Transform on each training sample, $I(k)$ is the periodogram at frequency $k$,   $r \in (0, 1)$ is the intended energy ratio. This subset $\mathcal{K}$ captures the dominant oscillatory components of the data. In practice, for sequences of Fourier length $M = 96$, this typically yields $K = 2$ to $5$ dominant frequencies across datasets as presented in Appendix~\ref{apdx:energy_distribution}. Compared to previous works~\cite{xu2023fits, yi2024frequency, zhou2022fedformer} choosing this set $\mathcal{K}$ heuristically on each dataset, we show that a fixed ratio $r$ can adapt to data signal-noise-ratio in Appendix~\ref{apdx:subsec:snr-r-K}.

\textbf{A Simple Linear Backbone}. Together, rather than model the entire matrix with $O(NM)$ complexity, we select only a small set of frequencies $\{\omega_1, \omega_2, \dots, \omega_K\}$ from $\{\omega_0, \omega_1, \dots, \omega_{\lfloor M/2 \rfloor\}}\}$. Then, we directly estimate the ES of the forecasting series $\mathbf{y}$ by constructing an estimation  $\hat{\mathbf{A}}_{\mathbf{y}}$:
\begin{equation}
[\hat{\mathbf{A}}^r_{{\mathbf{y}}},\hat{\mathbf{A}}_{{\mathbf{y}}}^i] = \boldsymbol{\Theta}_2 \, \operatorname{ReLU}\!\big( \boldsymbol{\Theta}_1 \, \mathbf{x} + \mathbf{b}_1 \big) + \mathbf{b}_2 \in \mathbb{R}^{N \times 2K},
\label{eq:nes_model_mlp}
\end{equation}

where $\hat{\mathbf{A}}^r_{{\mathbf{y}}}, \hat{\mathbf{A}}^i_{{\mathbf{y}}} \in \mathbb{R}^{N \times K}$ denote the real and imaginary parts of the amplitude matrix, respectively, and the complex-valued representation is given by $\hat{\mathbf{A}}_{{\mathbf{y}}} =\hat{\mathbf{A}}^r_{{\mathbf{y}}} + i \hat{\mathbf{A}}^i_{{\mathbf{y}}}$. Here $\boldsymbol{\Theta}_1, \boldsymbol{\Theta}_2, \mathbf{b}_1, \mathbf{b}_2$ are network parameters, and the output dimension is $2NK$ rather than $2NM$: the $j$-th column of $\hat{\mathbf{A}}_{\mathbf{y}}$ corresponds to the $j$-th selected frequency $\omega_{k_j}$, and $\mathbf{F}$ in Equation~\eqref{eq:main_theorem} is restricted to the same columns, $\mathbf{F}_{n,j}=e^{i\omega_{k_j} n}$, so frequencies outside $\mathcal{K}$ carry zero amplitude by construction. Our experiments show that this simple backbone is sufficient to achieve state-of-the-art performance with high efficiency; see Appendix~\ref{subsect:efficiency_analysis} for more experimental details.

\textbf{Probabilistic and Deterministic Time Series Forecasting}. Notably, Equation~\eqref{eq:discrete_evolutionary_rep} admits two interpretations depending on the treatment of the latent variables \(W_k\), enabling deterministic and probabilistic modeling across tasks. 
If we assume that the entire sequence is generated by a set of fixed oscillatory components (e.g., by fixing \(W_k = 1\) for all \(k\)), the representation reduces to a deterministic time–frequency decomposition of the given signal. 
In contrast, when \(W_k \sim \mathcal{CN}(0,1)\) are treated as random variables, the same formulation becomes a generative model for non-stationary stochastic processes. 
We summarize these two perspectives in Table~\ref{tab:interpretations}.

\begin{wraptable}{r}{0.55\linewidth}
\centering
\footnotesize	
\vspace{-10pt}
\setlength{\tabcolsep}{4pt} 
\caption{Deterministic vs. Probabilistic Interpretation.}
\label{tab:interpretations}
\begin{tabular}{lcc}
\toprule
\textbf{Interpretation} & $\mathbf{y}$  &  $\mathbf{w}$ \\
\midrule
Deterministic &
$\displaystyle \frac{1}{\sqrt{M}} (\hat{\mathbf{A}}_{\mathbf{y}} \odot \mathbf{F}) \mathbf{1}$ &
$W_k = 1,\ \forall k$ \\[2ex]
Probabilistic &
$\displaystyle \frac{1}{\sqrt{M}} (\hat{\mathbf{A}}_{\mathbf{y}} \odot \mathbf{F}) \mathbf{w}$ &
$W_k \sim \mathcal{CN}(\mu_{W_k},\sigma_{W_k})$  \\
\bottomrule
\end{tabular}
\end{wraptable}

In practice, $\mu_{W_k},\sigma_{W_k}$ are learnable constants for each selected frequency and channel, shared across the forecast horizon, so that the horizon dependence of the predictive distribution enters only through $\hat{\mathbf{A}}$ (Appendix~\ref{apdx:spectral_rv}). For probabilistic forecasting, sampling from a distribution incurs a higher computational cost due to the need to sample and store latent variables $\mathbf{w}$. In contrast, the deterministic decomposition avoids this overhead. It enables efficient, lightweight inference while providing a physically interpretable representation that directly encodes time-varying spectral energy through the amplitude $\mathbf{A}$.
It is worth mentioning that, the estimated evolutionary spectrum $\hat{\mathbf{A}}$ itself serves as a physically interpretable representation for various downstream tasks. We explain and evaluate this capability via an interpretability analysis in Section~\ref{subsec:interpretability_analysis} and a time series classification task in Appendix~\ref{subsec:exp:timeseries_classification}.

\subsection{Connections with DFT and STFT}


The proposed framework TimeES subsumes classical Fourier methods as special cases. When the amplitude spectrum is time-invariant, i.e., $A(n, \omega_k) = \widehat{X}_k$ for all $n$, DES reduces to the inverse DFT. When the signal is partitioned into non-overlapping blocks of length $L$ and $A(n, \omega_k)$ is constant within each block (i.e., piecewise time-invariant), DES recovers the inverse STFT with a rectangular window. Thus, DFT and STFT correspond to fixed, analytically computed amplitude spectra with deterministic weights ($W_k = 1$), whereas DES learns $A(n, \omega_k)$ and supports stochastic synthesis, we summarize how DES can degenerate to DFT and STFT.


\begin{wraptable}{r}{0.54\linewidth}
\centering
\vspace{-10pt}
\caption{Special cases of the $A(n, \omega_k)$ and $W_k$}
\label{tab:method_comparison}
\footnotesize
\begin{tabular}{lcc}
\toprule
Method & $A(n, \omega_k)$ & $W_k$ \\
\midrule
iDFT   & $(\mathbf{F} \mathbf{x})_k$ & $1$ \\[2pt]
 iSTFT  & $(\mathbf{F} \mathbf{x}^{(m)})_k,\; m = \lfloor n/L \rfloor$ & $1$ \\[2pt]
DES (ours) & $\text{MLP}_{n,k}(\mathbf{x})$ & $\mathcal{CN}(\mu_{W_k},\sigma_{W_k})$ \\
\bottomrule
\end{tabular}
\end{wraptable}
As summarized in Table~\ref{tab:method_comparison}, TimeES, DFT, and STFT are closely related. In fact, under the local stationarity assumption and with a normalized window, the continuous-time STFT approximates the evolutionary spectrum: $X_{\mathrm{STFT}}(t,\omega) \approx A(t,\omega)$, providing a localized estimate of its instantaneous spectral content. A theoretical justification is given in Appendix~\ref{apdx:proof:approx}.

\section{Experiments}
\label{sec:experiments}
To verify the performance of TimeES as both probabilistic and deterministic forecasting framework, we conduct experiments on three types of tasks: multivariate probabilistic forecasting, univariate and multivariate deterministic time series forecasting.  We provide details of the experiment in Appendix~\ref{apdx:experiment}.  Among these, probabilistic forecasting is particularly challenging due to the lack of generative models specifically designed for time series; as a result, current strong baselines are all based on architectures originally developed for computer vision~(e.g., diffusion models)~\cite{ye2025non}, which may not adequately account for non-stationary time series.

\subsection{Probabilistic Time Series Forecasting}
\label{subsec:probabilistic_forecasting}
\textbf{Setup}. We follow previous work~\cite{ye2025non} to assess the effectiveness of our proposed method, TimeES, in multivariate probabilistic forecasting task. We conduct comprehensive experiments across nine widely used real-world benchmarks including ETT\{m1,m2,h1,h2\}, ECL, EXG, Traffic, and Solar with input length 168 and forecast horizon 192 (36 on ILI). We assess both uncertainty quantification (CRPS, QICE) and forecasting accuracy (MSE and MAE) and compare against recent models including TimeGrad~\cite{rasul2021autoregressive}, CSDI~\cite{tashiro2021csdi}, TimeDiff~\cite{shen2023non}, DiffusionTS~\cite{yuan2024diffusion}, TMDM~\cite{li2024tmdm}, and NsDiff~\cite{ye2025non}. We set $M$ equal to the input length $N$ to cover all frequencies and set $r=0.9$ to capture the main energy frequency components. All other experiments follow this setting. For training, we sample 100 times and calculate the mean values to match the ground truth. We provide a detailed hyperparameter sensitivity analysis in Appendix~\ref{apdx:hyperparameter_sensitivity}.

\begin{table*}[htbp]
  \centering
      \setstretch{0.9}
  \footnotesize
  \caption{Experimental results on nine real-world datasets. \textbf{Bold face} indicates the best result, and \underline{underline} indicates the second-best result.}
    \begin{tabular}{ccccccccccc}
    \toprule
    Models & Datasets & ETTh1 & ETTh2 & ETTm1 & ETTm2 & ECL   & EXG   & ILI   & Solar & Traffic \\
    \midrule
    TimeGrad & CRPS  & 0.606  & 1.212  & 0.647  & 0.775  & 0.397  & 0.826  & 1.140  & \underline{0.293} & 0.407  \\
    (2021) & MSE   & 1.062  & 3.462  & 1.218  & 1.690  & 0.505  & 1.567  & 4.197  & 0.475  & 0.983  \\
    \midrule
    CSDI  & CRPS  & 0.492  & 0.647  & 0.524  & 0.817  & 0.577  & 0.855  & 1.244  & 0.432  & 1.418  \\
    (2022) & MSE   & 0.949  & 1.226  & 1.002  & 1.723  & 1.007  & 1.701  & 4.515  & 0.763  & 1.731  \\
    \midrule
    TimeDiff & CRPS  & 0.465  & 0.471  & 0.464  & 0.316  & 0.750  & 0.433  & 1.153  & 0.700  & 0.771  \\
    (2023) & MSE   & \underline{0.517} & \underline{0.456} & 0.537  & \underline{0.268} & 0.879  & 0.402  & 3.958  & 0.821  & 1.350  \\
    \midrule
    DiffusionTS & CRPS  & 0.603  & 1.168  & 0.574  & 1.035  & 0.633  & 1.251  & 1.612  & 0.470  & 0.668  \\
    (2024) & MSE   & 1.089  & 3.273  & 1.030  & 2.372  & 1.072  & 3.628  & 6.053  & 0.749  & 1.473  \\
    \midrule
    TMDM  & CRPS  & 0.452  & 0.383  & 0.375  & 0.289  & 0.461  & 0.336  & 0.967  & 0.350  & 0.557  \\
    (2024) & MSE   & 0.696  & 0.512  & 0.494  & 0.315  & 0.257  & 0.334  & 3.636  & 0.250  & 0.679  \\
    \midrule
    NsDiff & CRPS  & \underline{0.392} & \underline{0.358} & \underline{0.346} & \underline{0.256} & \underline{0.290} & \underline{0.324} & \textbf{0.806 } & 0.300  & \underline{0.378} \\
    (2025) & MSE   & 0.594  & 0.514  & \underline{0.488} & 0.281  & \underline{0.209} & \underline{0.300} & \textbf{2.846 } & \underline{0.242} & \underline{0.637} \\
    \midrule
    TimeES   & CRPS  & \textbf{0.349 } & \textbf{0.320 } & \textbf{0.310 } & \textbf{0.230 } & \textbf{0.205 } & \textbf{0.270 } & \underline{0.834} & \textbf{0.193 } & \textbf{0.245 } \\
    (ours) & MSE   & \textbf{0.474 } & \textbf{0.376 } & \textbf{0.371 } & \textbf{0.235 } & \textbf{0.180 } & \textbf{0.242 } & \underline{3.052} & \textbf{0.179 } & \textbf{0.456 } \\
    \bottomrule
    \end{tabular}%
  \label{tab:probabilistic_forecasting_results}%
\end{table*}%
\begin{figure*}[!th]
\begin{center}
\centerline{\includegraphics[width=0.9\linewidth]{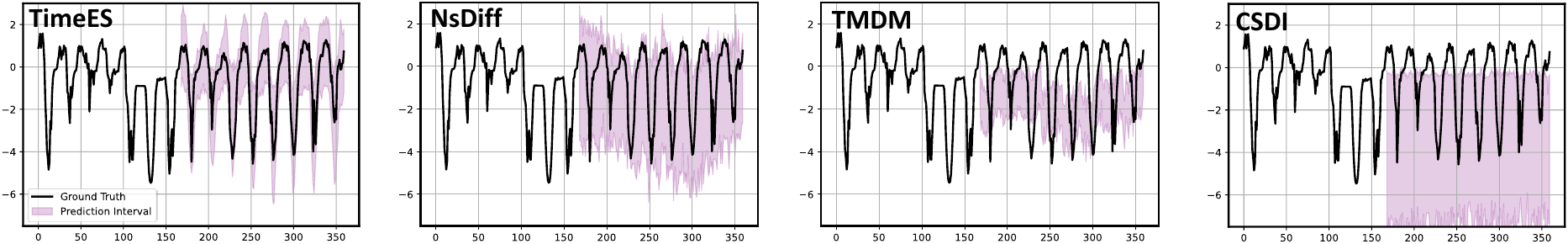}}
\caption{A probabilistic forecasting case of ETTh1 testing sample. We present 96\% prediction intervals as the purple areas.}
\vspace{-0.26in}
\label{fig:prob_show_case}
\end{center}
\end{figure*}

\textbf{Results}. We report CRPS and MSE in Table~\ref{tab:probabilistic_forecasting_results} and full results including QICE and MAE can be found in Appendix~\ref{apdx:full_other_results}. As shown in Table~\ref{tab:probabilistic_forecasting_results}, TimeES achieves great performance, surpassing the recent diffusion-based approach NsDiff specifically designed for non-stationarity, achieving state-of-the-art results in both deterministic prediction~(MSE) and uncertainty quantification~(CRPS). On average, TimeES achieves a 17.23\% reduction in CRPS and an 18.66\% reduction in MSE across all datasets. The improvement is particularly pronounced on the Traffic dataset, which exhibits strong non-stationarity~\cite{ye2024frequency, ye2025non} where TimeES reduces CRPS by 35.15\% and MSE by 28.34\%, demonstrating its ability to effectively model evolving spectral dynamics and non-stationary stochastic patterns. The only exception is the ILI dataset, where gains are marginal due to its limited sample size, which hinders reliable estimation of ES. Notably, TimeES achieves the highest efficiency in generating samples; refer to Appendix~\ref{subsect:efficiency_analysis} for full efficiency analysis.

\textbf{Sample Cases}. To provide a clearer understanding of TimeES’s capabilities, we visualize a representative prediction on the ETTh1 dataset. As shown in Figure~\ref{fig:prob_show_case}, compared to other baselines, TimeES effectively captures both the time-varying uncertainty and the underlying frequency dynamics. Specifically, TMDM and CSDI fail to model the non-stationarity between the input and output, resulting in poorly calibrated predictive variances. Although NsDiff better captures the non-stationary patterns of uncertainty, it still fails to identify the deterministic spectral components. In contrast, TimeES is the only model that accurately recovers both the deterministic frequency structure and the evolving uncertainty. 
We provide additional cases in Appendix~\ref{apdx:other_prob_showcases}.


\subsection{Deterministic Time Series Forecasting}
\label{subsec:univariate_and_multivariate_forecasting}
\textbf{Setup}. We follow previous work~\cite{liu2023itransformer} and evaluate on eight benchmarks: ETT\{m1,m2,h1,h2\}, ECL, EXG, Traffic, Weather, and Solar Energy. We use an input length of 96 and predict horizons of \{96, 192, 336, 720\}, reporting MSE and MAE.  To fully assess our model’s capability in modeling temporal dependency, we conduct an extra single variate forecasting experiments on each dataset by using only the first variable to predict its own future values. For multivariate forecasting, we train the model in a channel-independence setting~\cite{nie2022time} to independently predict each variate. We compare TimeES against the following representative baselines: Autoformer~\cite{wu2021autoformer}, FEDformer~\cite{zhou2022fedformer}, TimesNet~\cite{wu2022timesnet}, FITS~\cite{xu2023fits}, FreTS~\cite{yi2024frequency}, PatchTST~\cite{nie2022time}, and iTransformer~\cite{liu2023itransformer}. Among them, FEDformer, TimesNet, FITS, and FreTS are representative models that leverage FT or wavelet transforms for frequency modeling. We provide other additional baselines in Appendix~\ref{subsec:addtional_baselines}.

\begin{table*}[htbp]
  \centering
  \vspace{-8pt}
  \caption{Experimental results on forecasting, we report MSE. \textbf{Bold face} and \underline{underline} indicate the best and second-best results. S.V denotes average MSE results on single forecasting and M.V. denotes average MSE results on multivariate forecasting.}
    \setlength{\tabcolsep}{2pt} 
    \footnotesize
    \begin{tabular}{c|cc|cc|cc|cc|cc|cc|cc|cc}
    \toprule
      \multicolumn{3}{c}{Models \ \ Autoformer} & \multicolumn{2}{c}{FEDformer} & \multicolumn{2}{c}{TimesNet} & \multicolumn{2}{c}{FITS} & \multicolumn{2}{c}{FreTS} & \multicolumn{2}{c}{PatchTST} & \multicolumn{2}{c}{iTransformer} & \multicolumn{2}{c}{TimeES~(Ours)} \\
    Datasets & S.V.  & M.V.  & S.V.  & M.V.  & S.V.  & M.V.  & S.V.  & M.V.  & S.V.  & M.V.  & S.V.  & M.V.  & S.V.  & M.V.  & S.V.  & M.V. \\
    \midrule
    ETTm1 & 1.176  & 0.580  & 0.840  & 0.448  & 0.811  & 0.400  & 0.793  & 0.418  & \underline{0.755} & 0.414  & 0.763  & \textbf{0.386 } & 0.778  & 0.407  & \textbf{0.740 } & \underline{0.390} \\
    \midrule
    ETTm2 & 0.498  & 0.327  & 0.495  & 0.305  & 0.480  & 0.291  & 0.478  & 0.285  & \underline{0.452} & 0.323  & 0.462  & \underline{0.281} & 0.487  & 0.288  & \textbf{0.434 } & \textbf{0.278 } \\
    \midrule
    ETTh1 & 0.925  & 0.496  & 0.882  & \textbf{0.440 } & 0.929  & 0.458  & 1.155  & 0.534  & 1.010  & 0.504  & 0.970  & 0.469  & 0.968  & \underline{0.454} & \textbf{0.878 } & \textbf{0.440 } \\
    \midrule
    ETTh2 & 0.700  & 0.450  & 0.713  & 0.437  & 0.659  & 0.414  & 0.656  & 0.396  & \underline{0.596} & 0.500  & 0.633  & 0.387 & 0.682  & \underline{0.383}  & \textbf{0.583 } & \textbf{0.379 } \\
    \midrule
    ECL   & 0.679  & 0.227  & 0.610  & 0.214  & 0.591  & \underline{0.193} & 0.587  & 0.328  & \underline{0.530} & 0.205  & 0.543  & 0.216  & 0.601  & \textbf{0.178 } & \textbf{0.516 } & 0.197  \\
    \midrule
    EXG   & 0.827  & 0.613  & 0.765  & 0.519  & 0.711  & 0.416  & 0.762  & 0.389  & \underline{0.517} & 0.422  & 0.693  & 0.367  & 0.713  & \underline{0.360} & \textbf{0.356 } & \textbf{0.357 } \\
    \midrule
    Traffic & 1.019  & 0.628  & 0.566  & 0.610  & 0.564  & 0.620  & 1.164  & 0.827  & 0.576  & 0.555  & 0.521  & 0.555  & \underline{0.500} & \textbf{0.428 } & \textbf{0.431 } & \underline{0.478} \\
    \midrule
    Weather & 0.757  & 0.338  & 0.654  & 0.309  & 0.641  & 0.259  & 0.630  & 0.254  & \underline{0.523} & \underline{0.243} & 0.640  & 0.259  & 0.645  & 0.258  & \textbf{0.513 } & \textbf{0.239 } \\
    \midrule
    Solar & 3.046  & 0.885  & 0.252  & 0.292  & 0.216  & 0.301  & 0.330  & 0.403  & 0.214  & 0.255  & \textbf{0.191 } & 0.270  & \underline{0.200} & \textbf{0.233 } & 0.205  & \underline{0.243} \\
    \bottomrule
    \end{tabular}%
  \label{tab:forecasting_results}%
\end{table*}%

\textbf{Main Results}. We report MSE results in single variate and multivariate forecasting tasks in Table~\ref{tab:forecasting_results}, full results across steps are in Appendix~\ref{apdx:uniandmulti_results_apdx_exp}. Specifically, TimeES achieves state-of-the-art results on 8 out of 9 datasets in single variate forecasting. Notably, FreTS, an MLP-based model that leverages the STFT, and FEDformer, an architecture based on wavelet decomposition that also explicitly models time-frequency variations, are both outperformed by TimeES. On average, TimeES improves upon FreTS by 19.38\% and 16.00\%, and surpasses FEDformer by 9.98\% and 12.28\%, respectively, in single variate and multivariate settings. These gains highlight the importance of modeling frequency components at a finer granularity beyond conventional  representations. Also, TimeES achieves the highest efficiency in the inference step; refer to Appendix~\ref{subsect:efficiency_analysis} for efficiency analysis. We provide an additional synthetic experiment on a chirp signal in Appendix~\ref{apdx:chirp_synthetic_exp}.

\subsection{Ablation Study}
\begin{wraptable}{b}{0.55\linewidth}
  \centering
  \vspace{-13pt}
  \caption{Experimental results on different variants of TimeES on ETTh1 dataset. $\times$ indicates that the variant lacks the DES and is therefore not generative. P.F. denotes probabilistic forecasting.}
  \footnotesize
    \begin{tabular}{cccc}
    \toprule
    Variants & S.V.~(MSE) & M.V.~(MSE) & P.F.~(CRPS) \\
    \midrule
    w/o DES & 0.974±0.014 & 0.554±0.006 & $\times$ \\
    w/o E & 1.282±0.012 & 0.604±0.010 & 0.466±0.009 \\
    TimeES & \textbf{0.929±0.009} & \textbf{0.485±0.003} & \textbf{0.349±0.007} \\
    \bottomrule
    \end{tabular}%
  \label{tab:ablation_study}%
  \vspace{-2pt}
\end{wraptable}%
\textbf{Main Ablation.} To analyze the effectiveness of TimeES'components, we conduct ablation studies  across deterministic 720 steps forecasting, and probabilistic forecasting tasks on ETTh1 dataset. The experiment settings are consistent with Section~\ref{subsec:probabilistic_forecasting} and ~\ref{subsec:univariate_and_multivariate_forecasting}. Specifically, we include two types of variants: (1) \textbf{w/o E}~(without evolution): We enforce identical spectra at all forecasting steps to simulate a static spectral representation. (2) \textbf{w/o DES}~(without DES):  We replace the DES with a simple linear layer that directly outputs predictions.  The results are presented in Table~\ref{tab:ablation_study}, where it shows that TimeES achieves the best performance across all three tasks. Notably, although omitting DES increases prediction variance, it still yields a modest gain over w/o E: relative MSE improvements of 24.02\% and 8.27\% on single variate and multivariate long-term forecasting tasks, respectively. This further underscores the importance of modeling the evolutionary spectra. Full results are in Appendix~\ref{apdx:subsec:ablatino_exp}.

 \textbf{Plug-and-Play Testing}. We further evaluate TimeES as a plug-in enhancement module within a decomposition framework by setting $\mathbf{w} = \mathbf{1}$ (see Table~\ref{tab:interpretations}). Specifically, we replace the default backbone with PatchTST and iTransformer and compare their enhanced variants against the original models on the ETTh1 dataset in multivariate long-term forecasting. As shown in Table~\ref{tab:backbone_enhancement}, TimeES consistently improves all backbone models. When integrated with iTransformer, the average MSE reductions become increasingly pronounced as the horizon grows: 0.56\%, 1.88\%, 2.44\%, and 5.83\% for horizons 96, 192, 336, and 720, respectively. Full ETT results are provided in Appendix~\ref{apdx:full_other_results}.
\begin{table}[htbp]
  \centering
  \footnotesize
    \setlength{\tabcolsep}{5pt} 
  \caption{Backbone enhancement on the ETTh1 dataset, TimeES is integrated as a plug-in module.}
    \begin{tabular}{c|cccc|cc|cccc|cc}
    \toprule
    \multicolumn{1}{c}{Model} & \multicolumn{2}{c}{PatchTST} & \multicolumn{2}{c|}{+TimeES} & \multicolumn{2}{c}{Improvement} & \multicolumn{2}{c}{iTransformer} & \multicolumn{2}{c|}{+TimeES} & \multicolumn{2}{c}{Improvement} \\
    \multicolumn{1}{c}{Metrics} & MSE   & MAE   & MSE   & MAE   & MSE   & MAE   & MSE   & MAE   & MSE   & MAE   & MSE   & MAE \\
    \midrule
    96    & 0.414  & 0.419  & \textbf{0.385 } & \textbf{0.402 } & 6.89\% & 4.15\% & 0.386  & 0.405  & \textbf{0.384 } & \textbf{0.402 } & 0.56\% & 0.78\% \\
    192   & 0.460  & 0.445  & \textbf{0.434 } & \textbf{0.437 } & 5.65\% & 1.86\% & 0.441  & 0.436  & \textbf{0.433 } & \textbf{0.428 } & 1.88\% & 1.77\% \\
    336   & 0.501  & 0.466  & \textbf{0.451 } & \textbf{0.457 } & 9.99\% & 2.02\% & 0.487  & 0.458  & \textbf{0.475 } & \textbf{0.451 } & 2.44\% & 1.51\% \\
    720   & 0.500  & 0.488  & \textbf{0.482 } & \textbf{0.479 } & 3.68\% & 1.82\% & 0.503  & 0.491  & \textbf{0.474 } & \textbf{0.480 } & 5.83\% & 2.18\% \\
    \midrule
    Avg.   & 0.469  & 0.455  & \textbf{0.438 } & \textbf{0.444 } & 6.56\% & 2.42\% & 0.454  & 0.448  & \textbf{0.441 } & \textbf{0.440 } & 2.84\% & 1.59\% \\
    \bottomrule
    \end{tabular}%
  \label{tab:backbone_enhancement}%
\end{table}%

\subsection{Interpretability Analysis}
\label{subsec:interpretability_analysis}

 \begin{wrapfigure}{r}{0.5\linewidth}
 \centering
\vspace{-8pt}
\centerline{\includegraphics[width=1\linewidth]{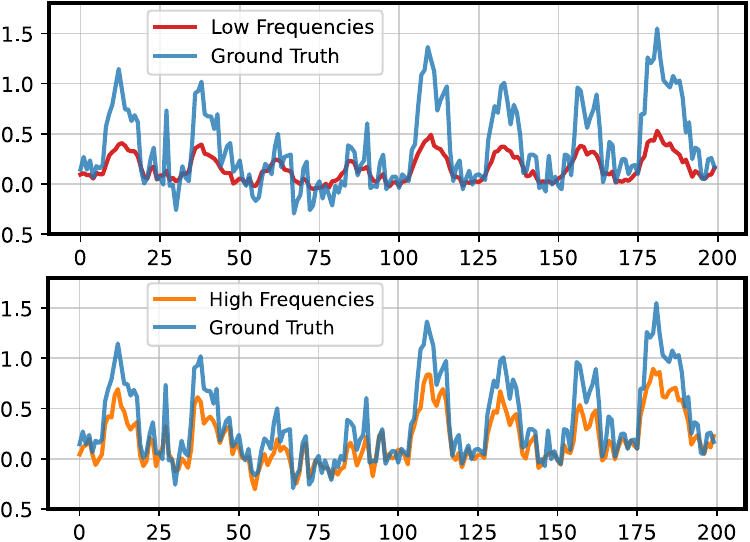}}
\caption{Low- and high-frequency components analysis on ETTh1 dataset.}
\label{fig:main_interpretability_analysis}
\end{wrapfigure}

Our learned ES can be used to reconstruct signals by separately synthesizing components from different frequency bands (like DFT), which offers strong interpretability. To investigate the interpretability of TimeES, we reconstruct high- and low-frequency signals using the learned ES and examine whether each component captures the expected characteristics. Specifically, we conduct an experiment on the ETTh1 dataset, with results visualized in Figure~\ref{fig:main_interpretability_analysis}. As shown, the high- and low-frequency components encode distinctly different information: the low-frequency component clearly captures the long-term trend, while the high-frequency component accurately reflects abrupt changes and fine-grained dynamics.  We provide more results in Appendix~\ref{apdx:interpretability_analysis}.

\section{Conclusion}
\label{sec:conclusion}
We introduce \textbf{TimeES}, a novel framework that bridges ES theory with deep learning for both probabilistic and deterministic forecasting. By exploiting Hermitian symmetry and spectra energy sparsity, we reduce  complexity from $\mathcal{O}(NM)$ to $\mathcal{O}(NK)$ with $K \ll M/2$. To the best of our knowledge, TimeES represents the first work to harness evolutionary spectra theory and employ deep neural networks for modeling non-stationary random process.


\textbf{Limitation \& Future work} Our frequency selection strategy may not be effective under high-noise conditions, as spectral energy may not concentrate at a few frequencies. Learning evolutionary spectra also requires sufficient data (Appendix~\ref{apdx:limited_data}). Additionally, the current implementation employs a standard architecture. Future work can focus on noise-robust frequency selection, limited-data learning, non-Gaussian or mixture spectral measures, and specialized architectures that better capture the temporal dynamics of evolutionary spectra while maintaining computational efficiency.

\bibliographystyle{plainnat}
\bibliography{nips2026_camera_ready}








\appendix
\clearpage

\section{Full Related Works}
\label{apdx:full_related_works}
\subsection{Evolutionary Spectra for Non-stationary Time Series}
Classical frequency analysis relies on the Fourier Transform (FT)~\cite{brigham1988fast}, which assumes a time-invariant spectrum, making it unsuitable signals with evolving frequency content. Time–frequency methods such as STFT~\cite{allen2005unified} and wavelet transforms~\cite{mallat1999wavelet} address this by localizing spectral estimation, but suffer from inherent trade-offs: STFT is limited by the Heisenberg uncertainty principle, while wavelets use fixed bases that lack adaptivity~\cite{flandrin1998time}. The Wigner–Ville distribution~\cite{wigner1932quantum} provides high-resolution estimates, but can introduce severe cross-terms in multi-component signals. To formally account for non-stationarity, Priestley proposed evolutionary spectra theory~\cite{priestley1965evolutionary,priestley1988spectral}, which defines a point-wise spectrum for harmonizable stochastic processes. It is intractable to determine the ES since the evolutionary spectrum is fundamentally non-identifiable and admits infinitely many valid solutions. There exist some techniques to estimate the ES, though these depend on very limiting conditions such as local stationarity or narrow-bandness ~\cite{schillinger2010accurate, spanos2004evolutionary, spanos2005stochastic, von1996wavelet}. Recent approaches rely on perturbation-based optimization methods~\cite{benowitz2015determining}, requiring a strict initialization. TimeES estimates evolutionary spectra directly from the data, without relying on strong parametric assumptions or constraints.

\subsection{Deep Learning for Deterministic Time Series Forecasting}
Time series is important in applications like traffic~\cite{dong2024heterogeneity, gao2023spatial, jin2023spatio}, and stock forecasting~\cite{cheng2022financial, duan2022factorvae}. Recent years have witnessed significant advances in deep learning for deterministic time series forecasting~\cite{cirstea2018correlated, jin2022multivariate, yang2024survey}, spanning forecasting~\cite{zhou2021informer}, classification~\cite{wu2022timesnet}, and anomaly detection~\cite{chandola2009anomaly}, etc. Architectures such as RNNs~\cite{connor1994recurrent, lin2025segrnn, sbrana2020n}, CNNs~\cite{liu2022scinet}, and Transformers~\cite{wu2021autoformer} have demonstrated strong empirical performance by capturing temporal dependencies through sequential or attention-based mechanisms~\cite{chen2024multi, jiang2023spatio, qiu2025duet, liu2023itransformer}. More recently, frequency-aware models like FEDformer~\cite{zhou2022fedformer} and FreTS~\cite{yi2024frequency} explicitly incorporate spectral components to improve long-term forecasting~\cite{yi2024filternet, yi2025survey}.  To further address non-stationarity, recent works explicitly model distributional or frequency shifts~\cite{fan2023dish, kim2021reversible, ye2024frequency, liu2024timebridge}. However, these approaches do not explicitly model spectral evolution, thus cannot effectively capture the frequency variations. 


\subsection{Deep Learning for Probabilistic Time Series Forecasting}
Currently, most time series generation methods are based on deep generative models originally developed for computer vision tasks. Early approaches are predominantly based on Generative Adversarial Networks~(GANs)~\cite{mogren2016c, yoon2019time}. Nevertheless, the instability of adversarial training has motivated a shift toward other alternative deep generative paradigms. TimeVAE~\cite{desai2021timevae} incorporates an interpretable temporal prior into a Variational Autoencoder (VAE) framework, Fourier Flows~\cite{alaa2021generative} use a normalizing flow architecture and apply a cascade of spectral filters to enable exact likelihood optimization. Denoising Diffusion Probabilistic Models (DDPMs), have recently been adapted to time series due to their strong empirical performance in modeling complex distributions.  CSDI~\cite{tashiro2021csdi} introduces a masking strategy that enables joint training for imputation and generation; DiffusionTS~\cite{yuan2024diffusion} combines the Fourier Transform with DDPMs to reconstruct and interpret time series in the spectral domain; and \citet{alcaraz2022diffusion} extends DDPMs with a structured spatial prior to capture long-term dependencies.  Among generative tasks, probabilistic forecasting stands out as especially challenging due to the non-stationarity inherent in real-world time series. In this context, DDPMs have also shown great promise~\cite{tyralis2022review}:  TimeGrad~\cite{rasul2021autoregressive} integrates an autoregressive diffusion process guided by hidden states of RNN; TimeDiff~\cite{shen2023non} stabilizes generation through future mixup and autoregressive initialization; and \citet{kollovieh2024predict} propose a self-guiding strategy based on structured state-space models. Moreover, building on the Schrödinger bridge framework of \citet{shi2023diffusion}, \citet{chen2023provably} provides theoretical convergence guaranties and algorithmic refinements for diffusion-based time series modeling. However, most of these methods assume stationary endpoint distributions, which are ill-suited for non-stationary temporal dynamics. To address this, NsDiff~\cite{ye2025non} incorporates a location-scale noise model into DDPM, allowing the endpoint distribution to vary over time and thereby better capture non-stationarity.  Nevertheless, nearly all existing approaches are fundamentally built upon architectures originally designed for computer vision and therefore implicitly rely on the assumption of i.i.d. data~\cite{li2025diffusion}. In contrast, our proposed TimeES is grounded in the explicit assumption that the underlying data-generating process is non-stationary and stochastic.

\section{Proof of Theorems}
\label{sec:proofs}
\subsection{Theorem \ref{theorem:spectral_representation}}
\label{apdx:proof:spectral_representation}
When analyzing the spectrum of a continuous-time stochastic process $\{X(t)\}_{t \in \mathbb{R}}$, the Fourier integral is not directly applicable. First, the Fourier series assumes that the signal is composed of a finite sum of periodic components, which is incompatible with general stochastic processes. Second, the Fourier integral requires that the sample paths be absolutely integrable over $(-\infty, \infty)$, i.e., $\int_{-\infty}^\infty |X(t)| \, dt < \infty$. This condition almost surely fails for stochastic processes, as their realizations typically do not decay at infinity and thus are not  absolutely integrable.

To circumvent this, we consider a truncated version of the process as:
\begin{equation}
X_T(t) = 
\begin{cases}
X(t), & -T \leq t \leq T, \\
0, & \text{otherwise}.
\end{cases}
\end{equation}
Assuming that this function is continuous between $-T$ and $T$, we can express $X_T$ in Fourier integral as:
\begin{equation}
X_T(t) = \frac{1}{\sqrt{2\pi}} \int_{-\infty}^{\infty} G_T(\omega) e^{i \omega t} \, d\omega,
\end{equation}
where
\begin{equation}
G_T(\omega) = \frac{1}{\sqrt{2\pi}} \int_{-T}^{T} X(t) e^{-i \omega t} \, dt.
\end{equation}

A natural idea is to take the limit $T \to \infty$ to recover a spectral representation of $X(t)$. However, since $T \to \infty$, $G_T(\omega)$ does not converge in the usual sense because $X(t)$ is not integrable. In fact, one can show that $|G_T(\omega)| = \mathcal{O}(\sqrt{T})$, implying divergence. Instead of focusing on energy (which is infinite), we consider \emph{power}, defined via the normalized expected squared magnitude as:
\begin{equation}
h(\omega) := \lim_{T \to \infty} \mathbb{E}\left[ \frac{|G_T(\omega)|^2}{2T} \right],
\end{equation}
assuming that this limit exists. The function $h(\omega)$ is interpreted as the \emph{spectral density} of the process.

Crucially, although $G_T(\omega)$ diverges as $T \to \infty$, its scaled increments remain well-behaved. Define the integrated Fourier transform as:
\begin{equation}
Z_T(\omega) = \frac{1}{\sqrt{2\pi}} \int_{-\infty}^{\omega} G_T(\theta) \, d\theta.
\end{equation}
Then, over a small frequency interval $\delta \omega_n = \omega_{n+1} - \omega_n$, the increment satisfies 
\begin{equation}
\Delta Z_T(\omega_n) = Z_T(\omega_{n+1}) - Z_T(\omega_n) \approx \frac{1}{\sqrt{2\pi}} G_T(\omega_n) \, \delta \omega_n.
\end{equation}
Since $|G_T(\omega_n)| = \mathcal{O}(\sqrt{T})$ and $\delta \omega_n \sim 1/T$, it follows that $\Delta Z_T(\omega_n) = \mathcal{O}(1)$, i.e., the increments remain bounded as $T \to \infty$.

This observation motivates replacing the ill-defined Fourier integral with a \emph{stochastic integral} with respect to a random measure. In the limit $T \to \infty$, the process $Z_T(\omega)$ converges (in an appropriate probabilistic sense) to a random function $Z(\omega)$ with orthogonal increments satisfying $\mathbb{E}[|dZ(\omega)|^2] = h(\omega) \, d\omega$.

Consequently, the original process admits the spectral representation:
\begin{equation}
X(t) = \int_{-\infty}^{\infty} e^{i t \omega} \, dZ(\omega),
\end{equation}
where $Z(\omega)$ is a complex-valued stochastic process with uncorrelated (or orthogonal) increments. This completes the justification of the spectral representation theorem.

\subsection{Theorem \ref{theorem:main_theorem}}
\label{section:proof:maintheorem}

First, for a discrete signal $\{X_t\}$, a generalized Cramér-type representation~\cite{cohen1995time} can be written as a time-varying Fourier--Stieltjes integral~\cite{priestley1965evolutionary} as:
\begin{equation}
X_n=\frac{1}{\sqrt{2 \pi}} \int_{-\pi}^\pi e^{i \omega n} d Z_n(\omega)
\end{equation}
where $d Z_n(\omega)$ is a complex-valued random measure with a covariance structure governed by the evolutionary spectral density. Under the evolutionary spectra framework defined in Theorem~\ref{theorem:spectral_representation} and ~\ref{theorem:evolutionary_spectra}, $dZ_n(w)$ admits:
\begin{equation}
d Z_n(\omega)=A(n, \omega) d W(\omega)
\end{equation}
where $A(n, \omega)$ is a complex amplitude modulation function, $d W(\omega)$ is a complex Gaussian white noise measure on $[-\pi, \pi)$ satisfying:
\begin{equation}
\mathbb{E}[d W(\omega)]=0, \quad \mathbb{E}[d W(\omega) \overline{d W(\nu)}]=\delta(\omega-\nu) d \omega
\end{equation}
where $\delta$ is the Dirac delta function that is zero everywhere except in $\nu = \nu_0$. Since numerical simulation cannot handle continuous integrals, we discretize the frequency axis into $M$ equally spaced points over $[0,2 \pi)$:
$$
\omega_k=\frac{2 \pi k}{M}, \quad k=0,1, \ldots, M-1.
$$

To align with the standard Fourier convention on $[-\pi, \pi)$, we may map high frequencies:

\begin{equation}
\omega_k \leftarrow \begin{cases}\omega_k & \text { if } \omega_k \leq \pi \\ \omega_k-2 \pi & \text { if } \omega_k>\pi\end{cases}
\end{equation}
and define the frequency spacing as:
\begin{equation}
\Delta \omega=\frac{2 \pi}{M}.
\end{equation}

We approximate the stochastic integral Equation~\eqref{eq:spectrum_representation_theorem} by a Riemann sum:

\begin{equation}
X_n \approx \frac{1}{\sqrt{2 \pi}} \sum_{k=0}^{M-1} A\left(n, \omega_k\right) e^{i \omega_k n} \cdot \Delta W_k,
\end{equation}

where $\Delta W_k$ approximates the increment $d W(\omega)$ over the interval $\left[\omega_k, \omega_k+\Delta \omega\right)$.

For a complex Gaussian white noise measure, the increments satisfy:

\begin{equation}
\mathbb{E}\left[\Delta W_k\right]=0, \quad \mathbb{E}\left[\Delta W_k \overline{\Delta W_{\ell}}\right]=\delta_{k \ell} \Delta \omega.
\end{equation}

Thus, we can write:

\begin{equation}
\Delta W_k=\sqrt{\Delta \omega} \cdot W_k,
\end{equation}

where $W_k \sim$ are i.i.d. standard complex normal random variables.



Substituting into Equation~\eqref{eq:evolutionary_spectra1}:

\begin{equation}
X_n \approx \frac{1}{\sqrt{2 \pi}} \sum_{k=0}^{M-1} A\left(n, \omega_k\right) e^{i \omega_k n} \cdot \sqrt{\Delta \omega} \cdot W_k.
\end{equation}

Recall $\Delta \omega=2 \pi / M$, so $\sqrt{\Delta \omega}=\sqrt{2 \pi / M}$. Therefore:

\begin{equation}
X_n \approx \frac{1}{\sqrt{2 \pi}} \cdot \sqrt{\frac{2 \pi}{M}} \sum_{k=0}^{M-1} A\left(n, \omega_k\right) W_k e^{i \omega_k n}=\frac{1}{\sqrt{M}} \sum_{k=0}^{M-1} A\left(n, \omega_k\right) W_k e^{i \omega_k n}.
\end{equation}

Since $\omega_k=2 \pi k / M$, we have $e^{i \omega_k n}=e^{i 2 \pi k n / M}$, yielding the final discrete synthesis formula:

\begin{equation}
X_n \approx \frac{1}{\sqrt{M}} \sum_{k=0}^{M-1} A\left(n, \omega_k\right) W_k e^{i 2 \pi k n / M},
\end{equation}
which completes the proof.

\subsection{Theorem \ref{theorem:hermitian}}
\label{section:proof:hermitian}
\begin{proof}
We prove that $ X_n \in \mathbb{R} $ for all $ n $ if and only if the Hermitian symmetry condition~\eqref{eq:hermitian} holds.

($\Rightarrow$) Suppose $ X_n \in \mathbb{R} $ for all $ n $. Then $ X_n = \overline{X_n} $. Define $ B(n, k) = A(n, \omega_k) W_k $. Then Equation~\eqref{eq:evolutionary_spectra1} becomes
\[
\overline{X_n} = \frac{1}{\sqrt{M}} \sum_{k=0}^{M-1} \overline{B(n, \omega_k)} \, e^{-i \omega_k n}.
\]
Since $ \omega_k = 2\pi k / M $, we have $ e^{-i \omega_k n} = e^{i (\omega_{(M - k) \bmod M}) n} $. Reindexing the sum with $ k' = (M - k) \bmod M $ (which is a permutation of $ \{0, \dots, M-1\} $), we obtain
\[
\overline{X_n} = \frac{1}{\sqrt{M}} \sum_{k'=0}^{M-1} \overline{B(n, \omega_{(M - k') \bmod M})} \, e^{i \omega_{k'} n}.
\]
Because $ X_n = \overline{X_n} $ and the complex exponentials $ \{e^{i \omega_k n}\}_{k=0}^{M-1} $ form an orthogonal basis over $ n = 0, \dots, M-1 $, the coefficients in the two expansions must be equal. Hence,
\[
B(n, \omega_k) = \overline{B(n, \omega_{(M - k) \bmod M})}, \quad \forall k,
\]
which is precisely the Hermitian symmetry condition~\eqref{eq:hermitian}.

($\Leftarrow$) Conversely, assume~\eqref{eq:hermitian} holds. Then
\begin{align*}
\overline{X_n}
&= \frac{1}{\sqrt{M}} \sum_{k=0}^{M-1} \overline{B(n, \omega_k)} \, e^{-i \omega_k n} \\
&= \frac{1}{\sqrt{M}} \sum_{k=0}^{M-1} B(n, \omega_{(M - k) \bmod M}) \, e^{i (\omega_{(M - k) \bmod M}) n} \quad \text{(by~\eqref{eq:hermitian})} \\
&= \frac{1}{\sqrt{M}} \sum_{k'=0}^{M-1} B(n, \omega_{k'}) \, e^{i \omega_{k'} n} \quad \text{(reindex with } k' = (M - k) \bmod M) \\
&= X_n.
\end{align*}
Thus $ X_n = \overline{X_n} $, so $ X_n \in \mathbb{R} $ for all $ n $.

Finally, for $ k = 0 $, condition~\eqref{eq:hermitian} gives $ B(n, \omega_0) = \overline{B(n, \omega_0)} $, so $ B(n, \omega_0) \in \mathbb{R} $. If $ M $ is even, then $ k = M/2 $ satisfies $ (M - k) \bmod M = k $, so the same argument implies $ B(n, \omega_{M/2}) \in \mathbb{R} $.
\end{proof}

\subsection{Continuous STFT Approximates ES under Local Stationarity}
\label{apdx:proof:approx}

Consider a real- or complex-valued signal $x(u)$ that admits a generalized time--frequency representation of the form
\begin{equation}
    x(u) = \int_{-\infty}^{\infty} A(u, \nu)\, e^{i \nu u}\, \frac{d\nu}{\sqrt{2\pi}},
    \label{eq:signal_rep}
\end{equation}
where $A(u, \nu)$ is a smooth complex-valued function interpreted as the \emph{local spectral amplitude}. Such a representation is commonly used in the analysis of non-stationary signals under the assumption of \emph{local stationarity}.

Let $w \in \{L^1(\mathbb{R}) \cap L^2(\mathbb{R})\}$ be a real-valued analysis window with effective support concentrated near the origin, and assume that it is normalized so that
$\int_{-\infty}^{\infty} w(s)\, ds = 1$. The short-time Fourier transform (STFT) of $x$ with respect to $w$ is defined by
\begin{equation}
    X_{\mathrm{STFT}}(t, \omega) 
    = \int_{-\infty}^{\infty} x(u)\, w(u - t)\, e^{-i \omega u}\, du.
    \label{eq:stft_def}
\end{equation}

Substituting the representation~\eqref{eq:signal_rep} into~\eqref{eq:stft_def} and interchanging the order of integration (justified, for instance, if $A \in L^2(\mathbb{R}^2)$ and $w$ decays sufficiently fast), we obtain
\begin{align}
    X_{\mathrm{STFT}}(t, \omega)
    &= \int_{-\infty}^{\infty} \left[ \int_{-\infty}^{\infty} A(u, \nu)\, e^{i \nu u}\, \frac{d\nu}{\sqrt{2\pi}} \right] w(u - t)\, e^{-i \omega u}\, du \nonumber \\
    &= \int_{-\infty}^{\infty} \int_{-\infty}^{\infty} A(u, \nu)\, w(u - t)\, e^{i (\nu - \omega) u}\, du\, \frac{d\nu}{\sqrt{2\pi}}.
    \label{eq:double_integral}
\end{align}

Introduce the change of variables $s = u - t$, i.e., $u = t + s$, so that $du = ds$. Then~\eqref{eq:double_integral} becomes
\begin{align}
    X_{\mathrm{STFT}}(t, \omega)
    &= \int_{-\infty}^{\infty} \int_{-\infty}^{\infty} A(t + s, \nu)\, w(s)\, e^{i (\nu - \omega)(t + s)}\, ds\, \frac{d\nu}{\sqrt{2\pi}} \nonumber \\
    &= \int_{-\infty}^{\infty} \left[ \int_{-\infty}^{\infty} w(s)\, e^{i (\nu - \omega) s}\, ds \right] A(t + s, \nu)\, e^{i (\nu - \omega) t}\, \frac{d\nu}{\sqrt{2\pi}}.
    \label{eq:after_substitution}
\end{align}

We now invoke the local stationarity assumption: over the effective support of the window $w(s)$ (i.e., where $|s|$ is small compared to the time scale of variation of $A$), the function $A(t + s, \nu)$ varies slowly with $s$. Formally, we assume $ A(t + s, \nu) \approx A(t, \nu) \, \text{for all } s \text{ such that } w(s) \text{ is non-negligible}$. Under this approximation, $A(t, \nu)$ can be treated as constant with respect to $s$ and factored out of the inner integral:
\begin{align}
    X_{\mathrm{STFT}}(t, \omega)
    &\approx \int_{-\infty}^{\infty} A(t, \nu)\, e^{i (\nu - \omega) t} \left[ \int_{-\infty}^{\infty} w(s)\, e^{i (\nu - \omega) s}\, ds \right] \frac{d\nu}{\sqrt{2\pi}} \nonumber \\
    &= \int_{-\infty}^{\infty} A(t, \nu)\, e^{i (\nu - \omega) t}\, \widehat{w}(\omega - \nu)\, \frac{d\nu}{\sqrt{2\pi}},
    \label{eq:approx_stft}
\end{align}
where $\widehat{w}(\xi) = \int_{-\infty}^{\infty} w(s) e^{-i \xi s}\, ds$ denotes the Fourier transform of the window. Equation~\eqref{eq:approx_stft} shows that, under local stationarity, the STFT is approximately a \emph{frequency-smoothed} version of $A(t, \nu)$, modulated by a linear phase term. The smoothing kernel is precisely the spectral profile $\widehat{w}$ of the analysis window, which induces unavoidable \emph{time--frequency smearing} (or leakage). If, in addition, $A(t, \nu)$ is approximately constant over the bandwidth of $\widehat{w}$, that is, over frequencies $|\nu - \omega| \lesssim B_w$, where $B_w$ is the effective bandwidth of the window, then the convolution in Equation~\eqref{eq:approx_stft} is dominated by the value at $\nu = \omega$. In this case,
\[
    X_{\mathrm{STFT}}(t, \omega) \approx A(t, \omega)\, \widehat{w}(0)\, e^{i (\omega - \omega)t} = A(t, \omega),
\]
since $\widehat{w}(0) = \int w(s)\, ds = 1$ by normalization. Thus, up to negligible smearing terms due to the finite support of $w$, the STFT recovers the local spectral amplitude, which completes the proof.

\medskip

The approximation improves as the window narrows in time (better temporal localization) \emph{and} $A(u, \nu)$ varies more slowly in $u$. However, due to the uncertainty principle, narrowing the window in time broadens $\widehat{w}$ in frequency, increasing spectral smearing. Hence, the accuracy of the approximation depends on a trade-off governed by the signal’s intrinsic time--frequency structure and the choice of window.

\subsection{Relationships between SNR, $r$ and $K$}
\label{apdx:subsec:snr-r-K}
Suppose a additive noise model for the observed time series: $x_n = s_n + \varepsilon_n, \, n = 0, 1, \dots, N-1.$ where  $s_n$ is the latent clean signal, $\varepsilon_n$ is zero-mean additive noise, assumed to be uncorrelated with $s_n$, and drawn from a standard normal distribution.
Let $\hat{X}_k$, $\hat{S}_k$, and $\hat{E}_k$ denote the discrete Fourier transforms of $x_n$, $s_n$, and $\varepsilon_n$, respectively. Then:
$$
\hat{X}_k = \hat{S}_k + \hat{E}_k.
$$
The total signal energy (by Parseval’s theorem~\cite{chen2001freeway}) is:

$$
E_s = \lVert s \rVert_2^2 = \sum_{n=0}^{N-1} s_n^2 = \sum_{k=0}^{N-1} \frac{1}{N} |\hat{S}_k|^2.
$$
The expected total noise energy is:

$$
E_n = \mathbb{E}[\lVert \varepsilon\rVert_2^2] = N \sigma^2.
$$
Thus, the signal-to-noise ratio is defined as:

$$
\mathrm{SNR} := \frac{E_s}{E_n} = \frac{\lVert s\rVert_2^2}{N \sigma^2}.
$$
The periodogram at frequency bin $k$ is:
$$
I(k) = \frac{1}{N} |\hat{X}_k|^2 = \frac{1}{N} |\hat{S}_k + \hat{E}_k|^2.
$$
Taking expectation, we obtain:

$$
\mathbb{E}[I(k)] = \frac{1}{N} |\hat{S}_k|^2 + \frac{1}{N} \mathbb{E}[|\hat{E}_k|^2],
$$
where the first term is the signal energy and the second is the noise energy, defined as $E_s$ and $E_n$, respectively. Now consider selecting the smallest set of frequencies such that the cumulative observed energy reaches a fraction $r \in(0,1)$ of the total:

$$
\sum_{k \in \mathcal{K}} I(k) \geq r \sum_{k=0}^{N-1} I(k)
$$
In expectation, and suppose that the selected frequencies exactly capture a fraction $r$ of the total energy, this becomes:
$$
\mathbb{E}\left[\sum_{k \in \mathcal{K}} I(k)\right] \approx r\left(E_s+E_n\right)
$$
Then:
$$
\mathbb{E}\left[\sum_{k \in \mathcal{K}} I(k)\right] \approx E_s + K \sigma^2,
$$
since the signal energy $E_s$ is concentrated in $K$ dominant frequency bins, and each selected bin contributes $\sigma^2$ of noise energy in expectation (due to the flat spectrum of white noise). The total expected observed energy is $E_s + M \sigma^2 = E_s + E_n$, where $E_n = M \sigma^2$ is the total noise energy. Enforcing the energy retention criterion yields:
$$
E_s + K \sigma^2 \approx r (E_s + M \sigma^2).
$$
Solving for $r$, we obtain:
$$
r \approx \frac{E_s + K \sigma^2}{E_s + M \sigma^2}.
$$
we substitute $E_s = \mathrm{SNR} \cdot M \sigma^2$ to get:
$$
r \approx \frac{\mathrm{SNR} \cdot M \sigma^2 + K \sigma^2}{\mathrm{SNR} \cdot M \sigma^2 + M \sigma^2}
= \frac{N \cdot \mathrm{SNR} + K}{M (\mathrm{SNR} + 1)}.
$$
Equivalently, solving for SNR gives:

$$
\mathrm{SNR} \approx \frac{M r - K}{M (1 - r)}.
$$
This relationship reveals why fixing $r$ leads to adaptive behavior, when the signal has a high SNR,  only a few frequencies are needed; when noise dominates, requiring a larger set of frequencies to capture sufficient signal content. Hence, this eliminates the need for manual tuning like previous methods, requiring only a lower-bound  of $r$ to balance computational overhead and performance. We also provide analysis in Appendix~\ref{apdx:energy_distribution} to justify the selection of $r$.
Thus, the energy-thresholding rule automatically adjusts the retained frequencies $K$ to the data’s SNR, without requiring prior knowledge about the frequency structure or the noise level.

\section{Experimental Details}
\label{apdx:experiment}

Table~\ref{tab:experiment_benchmarks} shows a summary of the benchmarks. As a supplementary evaluation, we perform a classification task using the learned representations to validate their effectiveness in capturing task-relevant features.

\begin{table*}[htbp]
  \centering
  \footnotesize
  \setlength{\tabcolsep}{1pt}
  \caption{Summary of experiment benchmarks.}
    \begin{tabular}{m{3cm}|m{4.7cm}|m{2cm}|c|m{2.8cm}}
    \toprule
    \multicolumn{1}{m{3cm}|}{\centering Tasks} & 
    \multicolumn{1}{m{4.5cm}|}{\centering Benchmarks} & 
    \multicolumn{1}{m{2.2cm}|}{\centering Metrics} & 
    Input &  
    \multicolumn{1}{m{2.8cm}}{\centering Output} \\ 
    \midrule
    
    Long-term Forecasting (Single variate and \newline Multivariate) & ETT\{m1, m2, h1, h2\}, ECL, EXG, Traffic, Weather, Solar & MSE, MAE & 96 & \{96, 192, 336, 720\} \\
    \midrule
    Probabilistic \newline Forecasting & ETT\{m1, m2, h1, h2\}, ECL, EXG, Traffic, ILI, Solar & CRPS, QICE, \newline MSE, MAE & 168 & 36 (ILI),  192 (others)  \\
    \midrule
    Classification & UEA (5 subsets) & Accuracy & 315-2500 & \multicolumn{1}{c}{-} \\
    \bottomrule
    \end{tabular}%
  \label{tab:experiment_benchmarks}%
\end{table*}%

\subsection{Datasets Details}

\textbf{ETT}\footnote{\url{https://github.com/zhouhaoyi/ETDataset}.} (Electricity Transformer Temperature): The electricity transformer datasets contain power load information and oil temperature as the target variable. \{ETTh1, ETTh2\} are 1-hour-level datasets, and \{ETTm1, ETTm2\} are 15-minute-level datasets. Each data point consists of seven features: one ``oil temperature'' (OT) target and six power load features.

\textbf{ECL}\footnote{\url{ https://github.com/laiguokun/multivariate-time-series-data}.} (Electricity Consuming Load): The electricity consumption dataset contains hourly electricity consumption (in Kwh) of 321 clients in two years. 

\textbf{EXG} (Exchange-Rate): The Exchange-Rate dataset contains daily exchange rates of eight foreign countries including Australia, British, Canada, Switzerland, China, Japan, New Zealand and Singapore ranging from 1990 to 2016. 

\textbf{Traffic}\footnote{\url{ http://pems.dot.ca.gov}.\label{fn:PEMS}} (Traffic): The  data contains 48 months (2015-2016) of hourly road occupancy rates (between 0 and 1) measured by 862 sensors on freeways in the San Francisco Bay area.

\textbf{Weather}\footnote{\url{ https://www.ncei.noaa.gov/data/local-climatological-data}.} (Weather): The weather dataset contains 21 different meteorological features of the U.S. in 4 years, including visibility, wind speed, etc. 

\textbf{Solar}\footnote{\url{https://www.nrel.gov/grid/solar-power-data.html}.} (Solar Energy): The solar power production dataset records the hourly energy output from 137 photovoltaic (PV) plants in Alabama over the year 2006.

\textbf{ILI}\footnote{\url{https://gis.cdc.gov/grasp/fluview/fluportaldashboard.html}.} (Influenza-like Illness): The ILI dataset contains weekly percentages of patients reported with influenza-like illness with multiple flu seasons containing 7 features. 

\textbf{UEA Classification}\footnote{\url{https://timeseriesclassification.com}.}: The UEA \& UCR Time Series Classification Archive~\cite{bagnall2018uea} is a widely used benchmark suite for evaluating time series classification algorithms. It was introduced to unify and standardize evaluation across the field, combining datasets from the former UCR (University of California, Riverside) and UEA (University of East Anglia) archives.

\subsection{Metrics}

We adopt the following standard metrics to evaluate model performance:

\textbf{Mean Absolute Error (MAE)} measures the average magnitude of absolute prediction errors:
    \[
        \text{MAE} = \frac{1}{N} \sum_{i=1}^{N} |y_i - \hat{y}_i|,
    \]
    where $y_i$ and $\hat{y}_i$ denote the ground-truth and predicted values for the $i$-th sample, respectively, and $N$ is the total number of samples.

\textbf{Mean Squared Error (MSE)} quantifies the average squared deviation between predictions and ground truth, penalizing larger errors more heavily:
    \[
        \text{MSE} = \frac{1}{N} \sum_{i=1}^{N} (y_i - \hat{y}_i)^2.
    \]

\textbf{Accuracy} is used for classification tasks and represents the proportion of correctly classified samples:
    \[
        \text{Accuracy} = \frac{1}{N} \sum_{i=1}^{N} \mathbb{I}(\hat{y}_i = y_i),
    \]
    where $\mathbb{I}(\cdot)$ is the indicator function that returns 1 if the prediction matches the true label and 0 otherwise.

\textbf{CRPS}: The continuous ranked probability score (CRPS)~\cite{matheson1976scoring}  measures the compatibility of a cumulative distribution function (CDF) $F$ with an observation $x$ as: 
\begin{equation}
\operatorname{CRPS}(F, x)=\int_{\mathbb{R}}(F(z)-\mathbb{I}\{x \leq z\})^2 \mathrm{~d} z,
\end{equation}
where $\mathbb{I}_{z < q}$ denotes the indicator function. By approximating the predictive cumulative distribution function (CDF) $F$ with its empirical counterpart,
$\hat{F}(z) = \frac{1}{S} \sum_{s=1}^S \mathbb{I}\left\{ x^{0,s} \leq z \right\},$
constructed from $S$ samples $x^{0,s} \sim F$, the CRPS can be directly estimated from the generated trajectories. In our experiments, we use samples $S = 100$ to approximate the distribution $F$.

\textbf{QICE}: The Quantile Interval Calibration Error (QICE)~\cite{han2022card} measures the discrepancy between the empirical coverage of predicted quantile intervals (QIs) and their nominal target coverage. Specifically, given $I$ non-overlapping quantile intervals that partition the predictive distribution into equal-probability regions (each with nominal coverage $1/I$), QICE computes the average absolute deviation between the observed proportion of ground-truth values falling within each interval and the ideal proportion $1/I$. 

Formally, let $\hat{y}_n^{\text{low}_m}$ and $\hat{y}_n^{\text{high}_m}$ denote the lower and upper bounds of the $m$-th quantile interval for the $n$-th sample. The empirical coverage of the $m$-th interval is:
\[
r_m = \frac{1}{N} \sum_{n=1}^N \mathbb{I}\left\{ \hat{y}_n^{\text{low}_m} \leq y_n \leq \hat{y}_n^{\text{high}_m} \right\},
\]
where $\mathbb{I}\{\cdot\}$ is the indicator function. QICE is then defined as:
\begin{equation}
    \operatorname{QICE} := \frac{1}{I} \sum_{m=1}^I \left| r_m - \frac{1}{I} \right|.
\end{equation}
Under perfect calibration (i.e., when the predicted distribution matches the true data distribution), we expect $r_m \approx 1/I$ for all $m$, yielding $\text{QICE} \to 0$. Following~\citet{li2024tmdm}, we partition the probability range into $I = 10$ equal-sized decile-based intervals to compute QICE.

\subsection{Spectral Energy Sparsity}
\label{apdx:energy_distribution}

A key observation from the spectral analysis of real-world time series is the phenomenon of spectral energy sparsity, which refers to the fact that most of the signal energy is concentrated in a small subset of frequency components, while the majority of frequencies contribute negligible power. This property is empirically validated across multiple benchmark datasets, as illustrated in Figure~\ref{fig:energy_distribution}. As shown in Figure~\ref{fig:energy_distribution}, for all tested datasets, including electricity demand (ETTm1, ETTm2), traffic flow, weather variables, exchange rates, and solar energy, the dominant spectral energy is highly localized in a few low-frequency bins. Specifically, we highlight the top frequencies accounting for 90\% of total spectral energy using pink bars, revealing that this critical energy concentration often resides within only a handful of frequency indices. For instance, in ETTm1, 90\% of the energy is captured by just six frequencies; similarly, ETTm2 and ETTh1 require only two-five frequency components to retain 90\% of signal energy. Even more strikingly, in the ExchangeRate dataset, over 80\% of energy lies in a single dominant frequency, indicating strong periodicity.  Across all datasets, the number of required frequencies grows slowly with increasing energy ratio,indicating that even at lower thresholds (e.g., 95\%), only a small fraction of frequencies are responsible for most of the signal dynamics. Notably, even when capturing 99\% of energy, the number of involved frequencies rarely exceeds 10–20 out of hundreds, demonstrating a high degree of frequency-wise energy concentration. The analysis together verify our choice of a fixed energy ratio of $r=0.9$ is effective across the experiments.

\begin{figure}[h]
\vskip -0.05in
\begin{center}
\centerline{\includegraphics[width=1\columnwidth]{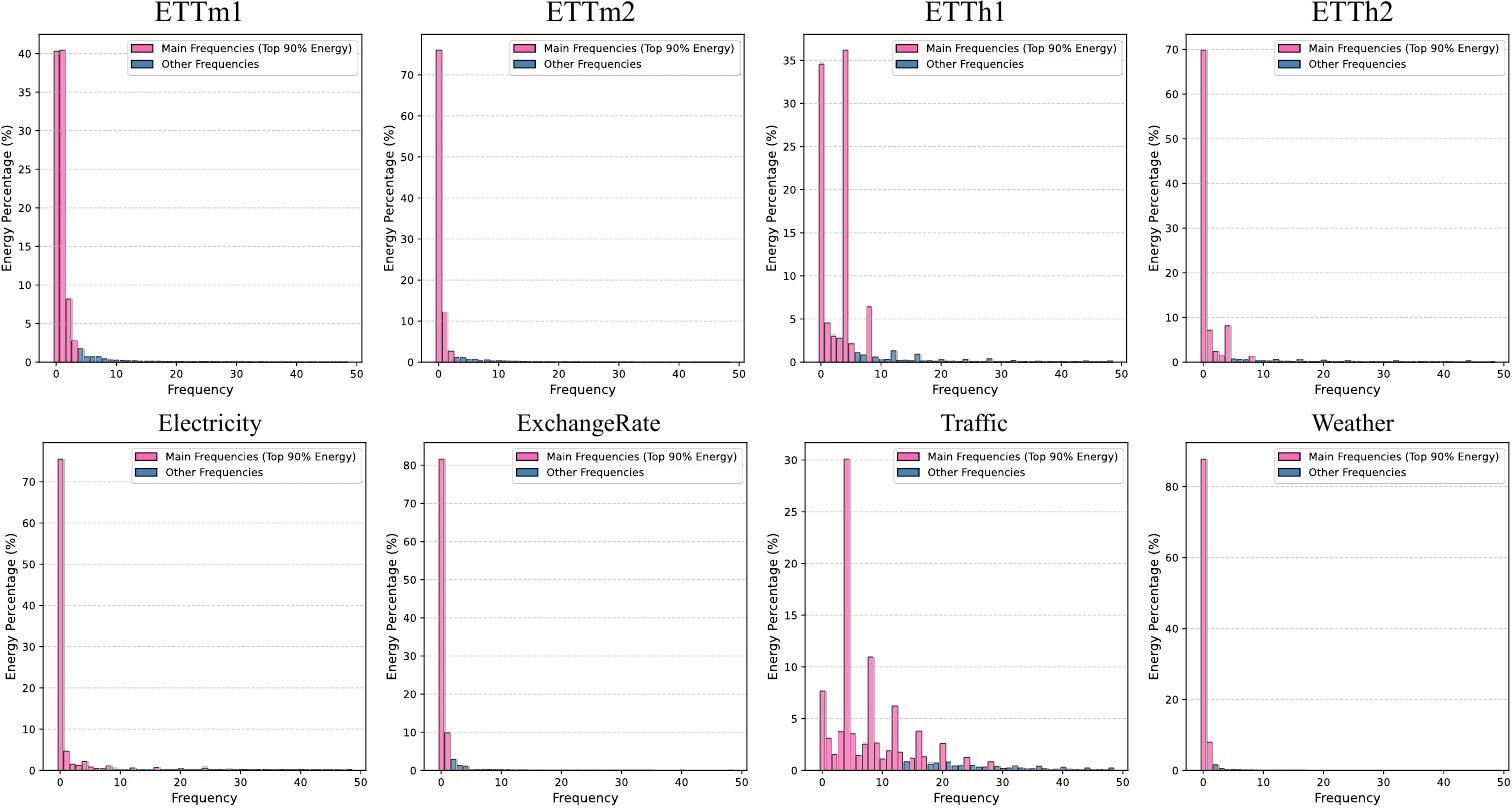}}
\caption{Energy distribution of different datasets. We use pink to highlight frequencies of top 90\% energy.}
\label{fig:energy_distribution}
\end{center}
\end{figure}

\subsection{Ablation Experiments}
\label{apdx:subsec:ablatino_exp}

To investigate the individual contributions of key components in TimeES, we conduct ablation studies on the ETTs datasets, focusing on two core modules: (1) the Evolutionary Spectra (E) module, which models the time-varying spectral structure; and (2) the DES module, responsible for generative modeling and uncertainty quantification. As shown in Table~\ref{tab:apdx:ablation_experiment_results}, removing the evolutionary spectra modeling (w/o E) leads to a significant performance degradation across all tasks. For instance, in single-variate long-term forecasting on ETTm1, MSE increases from 0.893 to 1.306, indicating that the model fails to capture the non-stationary dynamics when spectral evolution is ignored. Similar drops are observed in multi-variate settings and probabilistic forecasting, where CRPS rises from 0.310 to 0.417 without E. This confirms that explicitly modeling the time-dependent frequency content is crucial for accurate long-term prediction under changing periodicities. In contrast, removing the DES module (w/o DES) results in only minor performance changes in deterministic forecasting tasks, e.g., MSE increases by less than 5\% in most cases, suggesting that the core spectral representation remains effective even without the generative component. However, DES is essential for probabilistic forecasting: as shown in the bottom panel of Table~\ref{tab:apdx:ablation_experiment_results}, the model becomes non-generative without DES, rendering it incapable of producing predictive distributions or uncertainty estimates. Moreover, DES provides an interpretable frequency evolution matrix that enables fine-grained analysis of how dominant frequencies shift over time, a capability absent in standard Fourier-based methods. This not only enhances model transparency but also supports downstream applications such as anomaly detection and signal diagnosis.

As shown in Table~\ref{tab:apdx:ablation_experiment_results}, modeling the evolutionary spectra yields performance gains even without the DES. Notably, DES is a novel generative module; without it, the model cannot perform generative tasks such as probabilistic forecasting. Furthermore, DES provides an interpretable frequency evolution matrix that enables effective spectral analysis. These results collectively verify the importance of explicitly modeling the spectra at a finer-grained level.

\begin{table}[htbp]
  \centering
  \footnotesize
  \caption{Ablation results on ETTs datasets. We report MSE, MAE for 720-steps single- and multi-variate long-term forecasting and report CRPS, MSE for 192 steps multivariate probabilistic forecasting. The best result is denoted in \textbf{bold} characters.}
  \setlength{\tabcolsep}{12pt}
    \begin{tabular}{c|cccccccc}
    \toprule
    \multicolumn{9}{c}{Single-Variate Long-term Forecasting} \\
    \midrule
    Datasets & \multicolumn{2}{c|}{ETTm1} & \multicolumn{2}{c|}{ETTm2} & \multicolumn{2}{c|}{ETTh1} & \multicolumn{2}{c}{ETTh2} \\
    \midrule
    Variants & MSE   & \multicolumn{1}{c|}{MAE} & MSE   & \multicolumn{1}{c|}{MAE} & MSE   & \multicolumn{1}{c|}{MAE} & MSE   & MAE \\
    \midrule
    TimeES & \textbf{0.893} & \multicolumn{1}{c|}{\textbf{0.619}} & \textbf{0.590} & \multicolumn{1}{c|}{\textbf{0.553}} & \textbf{0.929} & \multicolumn{1}{c|}{\textbf{0.670}} & \textbf{0.618} & \textbf{0.585}  \\
    w/o E & 1.306  & \multicolumn{1}{c|}{0.792} & 0.673  & \multicolumn{1}{c|}{0.587} & 1.282  & \multicolumn{1}{c|}{0.840} & 0.785  & 0.676  \\
    w/o DES & 0.966  & \multicolumn{1}{c|}{0.662} & 0.639  & \multicolumn{1}{c|}{0.569} & 0.974  & \multicolumn{1}{c|}{0.699} & 0.744  & 0.646  \\
    \midrule
    \multicolumn{9}{c}{Multi-Variate  Long-term Forecasting} \\
    \midrule
    Datasets & \multicolumn{2}{c|}{ETTm1} & \multicolumn{2}{c|}{ETTm2} & \multicolumn{2}{c|}{ETTh1} & \multicolumn{2}{c}{ETTh2} \\
    \midrule
    Variants & MSE   & \multicolumn{1}{c|}{MAE} & MSE   & \multicolumn{1}{c|}{MAE} & MSE   & \multicolumn{1}{c|}{MAE} & MSE   & MAE \\
    \midrule
    TimeES & \textbf{0.465} & \multicolumn{1}{c|}{\textbf{0.445}} & \textbf{0.396} & \multicolumn{1}{c|}{\textbf{0.395}} & \textbf{0.485} & \multicolumn{1}{c|}{\textbf{0.479}} & \textbf{0.426} & \textbf{0.443} \\
    w/o E & 0.601  & \multicolumn{1}{c|}{0.510} & 0.414  & \multicolumn{1}{c|}{0.407} & 0.604  & \multicolumn{1}{c|}{0.538} & 0.439  & 0.452  \\
    w/o DES & 0.467  & \multicolumn{1}{c|}{0.447} & 0.411  & \multicolumn{1}{c|}{0.410} & 0.554  & \multicolumn{1}{c|}{0.508} & 0.432  & 0.449  \\
    \midrule
    \multicolumn{9}{c}{Multi-Variate  Probabilistic Forecasting} \\
    \midrule
    Datasets & \multicolumn{2}{c|}{ETTm1} & \multicolumn{2}{c|}{ETTm2} & \multicolumn{2}{c|}{ETTh1} & \multicolumn{2}{c}{ETTh2} \\
    \midrule
    \multicolumn{1}{c}{Variants} & CRPS  & \multicolumn{1}{c|}{MSE} & CRPS  & \multicolumn{1}{c|}{MSE} & CRPS  & \multicolumn{1}{c|}{MSE} & CRPS  & MSE \\
    \midrule
    TimeES & \textbf{0.310} & \multicolumn{1}{c|}{\textbf{0.371}} & \textbf{0.230} & \multicolumn{1}{c|}{\textbf{0.235}} & \textbf{0.349} & \multicolumn{1}{c|}{\textbf{0.474}} & \textbf{0.320} & \textbf{0.376} \\
    w/o E & 0.417  & \multicolumn{1}{c|}{0.378} & 0.289  & \multicolumn{1}{c|}{0.250} & 0.466  & \multicolumn{1}{c|}{0.592} & 0.363  & 0.387  \\
    w/o DES & \multicolumn{8}{c}{The model is not generative without DES.} \\
    \bottomrule
    \end{tabular}%
  \label{tab:apdx:ablation_experiment_results}%
\end{table}%

\subsection{Additional Baseline Experiments}
\label{subsec:addtional_baselines}
To provide a comprehensive evaluation of our proposed method, we conduct extensive experiments against six state-of-the-art baseline models spanning different architectural paradigms in time series forecasting. These include: KNF~\cite{wang2022koopman}, a Koopman neural operator-based approach; Koopa~\cite{liu2024koopa}, which leverages Koopman spectrum analysis; SKOLR~\cite{zhang2025skolr}, a structured Koopman learning framework; TimeBridge~\cite{liu2024timebridge}, a bridge-based temporal modeling architecture; TimeMixer~\cite{wang2024timemixer}, a channel-independent mixing strategy; and TimeXer~\cite{wang2024timexer}, a Transformer-based temporal encoder. This diverse set of baselines enables rigorous comparison across spectral methods, neural operators, and attention-based architectures, ensuring a thorough assessment of our method's effectiveness in multivariate time series forecasting tasks.

\begin{table}[htbp]
\centering
\footnotesize
\caption{Mean Squared Error (MSE) results on deterministic multivariate forecasting tasks across different datasets and prediction horizons. \textbf{Bold} values indicate the best performance for each setting. The last row shows the number of times each model achieves the best performance.}
\label{tab:mv_forecasting_results}
\setlength{\tabcolsep}{7pt}
\begin{tabular}{l|c|ccccccc}
\hline
\textbf{Dataset} & \textbf{Horizon} & \textbf{KNF} & \textbf{Koopa} & \textbf{SKOLR} & \textbf{TimeBridge} & \textbf{TimeMixer} & \textbf{TimeXer} & \textbf{TimeES} \\
\hline
\multirow{4}{*}{ETTm1} & 96 & 0.691 & 0.331 & 0.342 & \textbf{0.324} & 0.344 & 0.387 & 0.328 \\
 & 192 & 0.707 & 0.377 & 0.397 & 0.365 & 0.368 & 0.441 & \textbf{0.364} \\
 & 336 & 0.718 & 0.402 & 0.445 & 0.395 & \textbf{0.394} & 0.514 & 0.402 \\
 & 720 & 0.742 & 0.466 & 0.536 & 0.457 & \textbf{0.455} & 0.592 & 0.465 \\
\hline
\multirow{4}{*}{ETTm2} & 96 & 0.230 & 0.182 & 0.194 & 0.178 & 0.178 & 0.364 & \textbf{0.176} \\
 & 192 & 0.274 & 0.245 & 0.256 & 0.245 & 0.246 & 0.526 & \textbf{0.240} \\
 & 336 & 0.385 & 0.301 & 0.321 & 0.307 & 0.305 & 0.870 & \textbf{0.300} \\
 & 720 & 0.479 & 0.402 & 0.423 & 0.411 & 0.402 & 0.902 & \textbf{0.396} \\
\hline
\multirow{4}{*}{ETTh1} & 96 & 0.700 & 0.394 & 0.427 & 0.378 & 0.376 & 0.438 & \textbf{0.374} \\
 & 192 & 0.717 & 0.442 & 0.470 & 0.431 & \textbf{0.415} & 0.497 & 0.430 \\
 & 336 & 0.723 & 0.489 & 0.526 & 0.478 & \textbf{0.446} & 0.528 & 0.470 \\
 & 720 & 0.716 & 0.482 & 0.531 & 0.533 & \textbf{0.472} & 0.639 & 0.485 \\
\hline
\multirow{4}{*}{ETTh2} & 96 & 0.362 & 0.305 & 0.330 & 0.296 & \textbf{0.287} & 0.583 & 0.293 \\
 & 192 & 0.430 & 0.392 & 0.419 & 0.376 & 0.372 & 0.590 & \textbf{0.371} \\
 & 336 & 0.478 & 0.426 & 0.443 & 0.417 & 0.429 & 0.612 & \textbf{0.427} \\
 & 720 & 0.506 & 0.447 & 0.463 & 0.427 & 0.447 & 0.652 & \textbf{0.426} \\
\hline
\textbf{1st Count} & & 0 & 0 & 0 & 1 & 5 & 0 & 9 \\
\hline
\end{tabular}
\end{table}

\subsection{Synthetic Experiments}
\label{apdx:chirp_synthetic_exp}

To evaluate the ability of TimeES to capture \textit{spectral evolution}, a key challenge in non-stationary time series, we conduct a synthetic experiment on a chirp signal, which exhibits continuously varying frequency and amplitude over time. The signal is defined as:
\[
x(t) = A(t) \cdot \sin(2\pi f(t) t),
\]
where $A(t)$ and $f(t)$ are slowly increasing functions, simulating real-world phenomena such as radar signals or mechanical vibrations with changing dynamics, we let the periodicity change from 64 to 128, the magnitude change from 1 to 3, and other experiment settings are consistent with the deterministic forecasting experiment 96 step forecasting. As shown in Figure~\ref{fig:chirp_signal_exp_fig}(b), the training set contains only the initial portion of the chirp signal, where the frequency increases gradually. The testing set extends into a region with higher frequencies not seen during training, thus forming a challenging extrapolation task that requires modeling the \textit{evolution} of spectral components rather than mere interpolation. Figure~\ref{fig:chirp_signal_exp_fig}(c) visualizes the forecasting results across different models. While most methods (e.g., PatchTST, iTransformer, TimesNet) fail to track the increasing frequency and produce oscillations with fixed periodicity, FreTS shows partial success due to its frequency-aware design but still suffers from phase drift. In contrast, TimeES (NES) accurately captures both the evolving frequency and amplitude, producing a smooth and faithful reconstruction of the true signal. The quantitative comparison in Figure~\ref{fig:chirp_signal_exp_fig}(a) further confirms this advantage: TimeES achieves the lowest MSE ($0.290$), significantly outperforming all baselines. Notably, even Fourier-domain methods like FreTS, despite their explicit frequency modeling, struggle to maintain accuracy under spectral evolution, highlighting the limitation of static spectral bases. 
This synthetic experiment demonstrates that TimeES's learned evolutionary spectra can effectively model non-stationary dynamics by adapting to changing frequency content, providing strong empirical support for its capability in capturing long-term spectral trends.

\begin{figure}[h]
\begin{center}
\centerline{\includegraphics[width=1\linewidth]{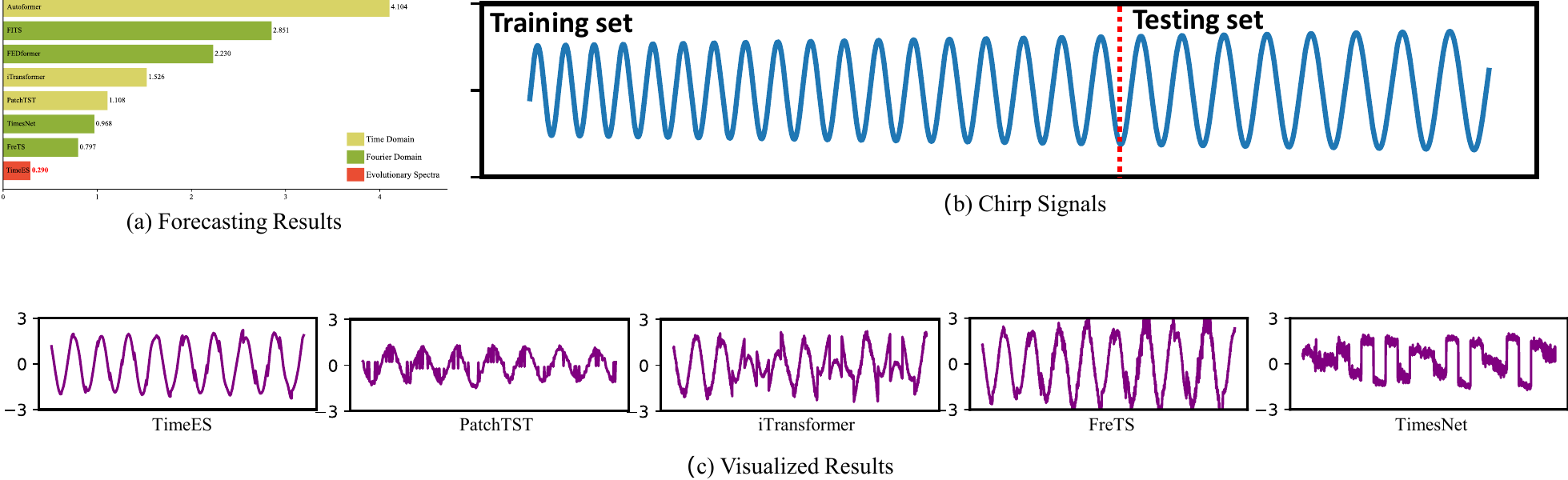}}
\caption{(a-f) Predictive uncertainty quantification by TimeES across different datasets. We plot the last 500 predicted time steps of the first four variables in each dataset. The orange line shows the ground truth, the blue line represents the predicted mean, and the gray shaded region indicates the 96\% prediction interval.}
\label{fig:chirp_signal_exp_fig}
\end{center}
\end{figure}

\subsection{Time Series Classification}
\label{subsec:exp:timeseries_classification}

We evaluate the discriminative power of the learned spectral representation $\mathbf{A}$ on five multivariate time series classification tasks from the UEA archive~\cite{bagnall2018uea}, as shown in Table~\ref{tab:classification} and summarized in Figure~\ref{fig:classificatoin_results}. Our method, TimeES, achieves an average accuracy of 62.96\%, outperforming the previous best model, TimesNet, by 2.85\%, demonstrating the effectiveness of modeling non-stationary spectral evolution for general-purpose time series representation. Notably, TimeES achieves state-of-the-art performance across multiple datasets: it reaches 92.83\% on \texttt{SelfRegulationSCP1}, a dataset with complex brainwave dynamics, and 36.81\% on EthanolConcentration, where fine-grained frequency patterns are critical. On UWaveGestureLibrary, our method performs competitively (86.50\%) despite being slightly behind iTransformer (87.50\%), suggesting that while temporal context is important, evolutionary spectral modeling captures complementary information. While other frequency-domain models like FreTS (59.04\%) and FEDformer (56.17\%) show improved performance over pure time-domain methods, they still fall short of TimeES. This suggests that simply applying frequency analysis is insufficient; the key lies in modeling how spectral support evolves over time, especially under non-stationary conditions such as changing periodicity. The performance of TimeES confirms that learning an adaptive and temporally coherent spectral representation not only benefits long-term forecasting but also enables more robust and transferable features for downstream tasks.

\begin{figure}[!h]
\vskip -0.05in
\begin{center}
\centerline{\includegraphics[width=0.8\linewidth]{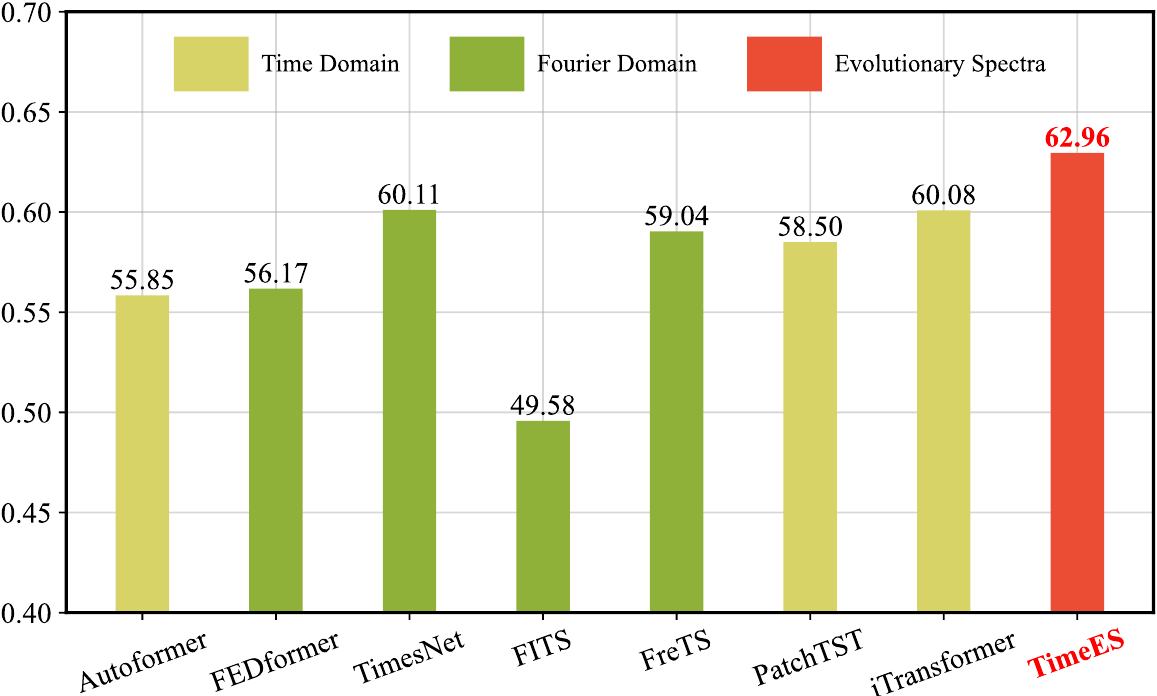}}
\vskip -0.1in 
\caption{Average accuracy across datasets.}
\label{fig:classificatoin_results}
\end{center}
\vskip -0.3in 
\end{figure}

\begin{table}[htbp]
  \centering
  \footnotesize
    \setlength{\tabcolsep}{4pt}
\caption{Classification accuracy on five datasets. \textbf{Bold} and \underline{underline} indicate best and second-best results respectively.}
    \begin{tabular}{ccccccccc}
    \toprule
          & Autoformer & FEDformer & TimesNet & FITS  & FreTS & PatchTST & iTransformer & TimeES \\
    \midrule
    EthanolConcentration & 31.60 & 31.20 & 35.70 & 25.10 & 24.79 & \underline{36.42} & 31.60 & \textbf{36.81} \\
    MotorImagery & 48.40 & 45.00 & 50.40 & 49.60 & 50.40 & \underline{52.80} & \textbf{55.60} & 52.00 \\
    SelfRegulationSCP1 & 84.00 & 88.70 & \underline{91.80} & 72.70 & 89.56 & 81.78 & 88.39 & \textbf{92.83} \\
    StandWalkJump & 29.33 & 30.67 & 37.33 & 38.67 & \underline{45.33} & 41.33 & 37.33 & \textbf{46.67} \\
    UWaveGestureLibrary & 85.90 & 85.30 & 85.30 & 61.81 & 85.10 & 80.19 & \textbf{87.50} & \underline{86.50} \\
    \midrule
    Avg.  & 55.85 & 56.17 & \underline{60.11} & 49.58 & 59.04 & 58.50 & 60.08 & \textbf{62.96} \\
    \bottomrule
    \end{tabular}%
    \label{tab:classification}%
\end{table}%


\section{Efficiency Analysis} 
\label{subsect:efficiency_analysis}
To evaluate the computational and memory efficiency of TimeES, we conduct an efficiency analysis on the ETTh1 dataset  probabilistic forecasting and 720-step long-term multivariate  forecasting. As shown in Figure~\ref{fig:efficiency_analysis}, TimeES achieves the highest efficiency among all compared methods in both generative and deterministic forecasting tasks. Notably, our probabilistic forecasting significantly outperforms prior diffusion-based approaches, which typically require hundreds or thousands of iterative denoising steps to generate predictions. In contrast, TimeES produces accurate samples in a single forward pass. Moreover, unlike frequency-domain models such as FITS that rely on explicit operations (e.g., DFT and iDFT) to compute spectral representations, TimeES directly learns time-series dynamics through a compact neural network without any explicit spectral transformation, which enables TimeES to surpass even highly optimized frequency-based models in deterministic forecasting speed, demonstrating its high efficiency.

\begin{figure}[!h]
\vskip -0.05in
\begin{center}
\centerline{\includegraphics[width=0.8\columnwidth]{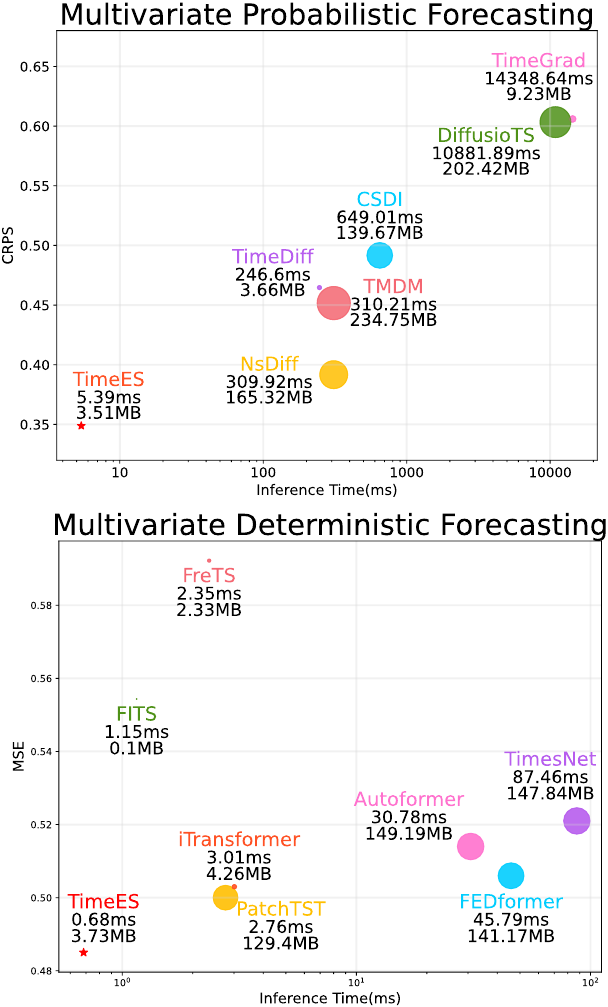}}
\vskip -0.1in 
\caption{Efficiency analysis on ETTh1 dataset.}
\label{fig:efficiency_analysis}
\end{center}
\vskip -0.3in 
\end{figure}






\section{Hyperparameter Sensitivity}
\label{apdx:hyperparameter_sensitivity}

We introduce a hyper-parameter $r$ to retain the frequency components that collectively account for the top $r$ ratio of total energy. The sensitivity analysis of this hyper-parameter is provided in Figure~\ref{fig:hyperparameter_sensitivity}. We observe that our empirically selected value $r = 0.9$ consistently yields stable performance across various forecasting tasks. Notably, on the Traffic dataset, aggressive frequency pruning leads to significant performance degradation, highlighting the importance of modeling multiple frequency components, which is consistent with the broad energy distribution observed in Figure~\ref{fig:energy_distribution}. Moreover, our energy-based selection strategy reliably captures the dominant frequencies, in contrast to naive top-$k$ approaches that may inadvertently miss critical spectral components due to rigid cardinality constraints. Importantly, on several datasets like ETTs multivariate in forecasting, appropriately selecting frequencies not only avoids performance loss but actually improves forecasting accuracy, further demonstrating that effective frequency selection can enhance model performance to a certain extent.

\begin{figure}[!h]
\begin{center}
\centerline{\includegraphics[width=1\columnwidth]{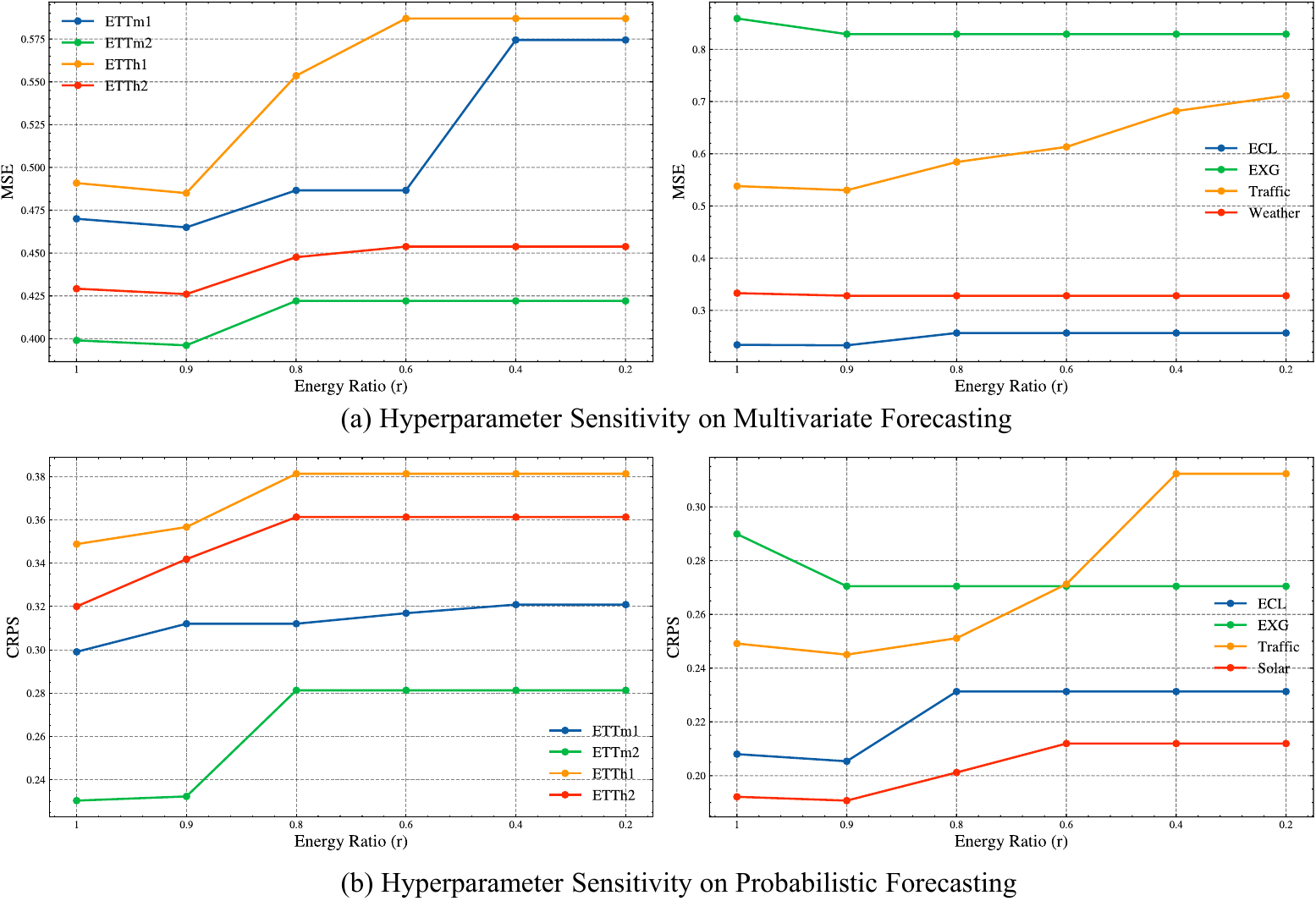}}
\caption{We conduct a sensitivity analysis of the hyper-parameter $r$ in multivariate long-term forecasting and probabilistic forecasting tasks. Specifically, we set the prediction horizon to 720 time steps. We report the Mean Squared Error (MSE) for deterministic forecasting and the Continuous Ranked Probability Score (CRPS) for probabilistic forecasting, evaluated on the ETT (ETTm1, ETTm2, ETTh1, ETTh2), ECL, EXG, Traffic, Weather, and Solar datasets.}
\label{fig:hyperparameter_sensitivity}
\end{center}
\end{figure}






\clearpage
\section{Additional Experimental Results}
\label{apdx:full_other_results}

Due to page limits, we place some full results of our experiments in the following: classification in Table~\ref{tab:classification},  probabilistic forecasting in Table~\ref{tab:probabilistic forecasting1} and Table~\ref{tab:probabilistic forecasting2}, single variable deterministic forecasting in Table~\ref{tab:uni_results},  multivariate deterministic forecasting results in Table~\ref{tab:multi_results}, backbone enhancement results in Table~\ref{tab:decompo_results}. The statistical significance is within 3\%.

\begin{table}[htbp]
  \centering
    \footnotesize

  \caption{Probabilistic forecasting performance (CRPS, QICE) across nine datasets.}
  \setlength{\tabcolsep}{5.5pt}
    \begin{tabular}{ccccccccccc}
    \toprule
    Models & Datasets & ETTh1 & ETTh2 & ETTm1 & ETTm2 & ECL   & EXG   & ILI   & Solar & Traffic \\
    \midrule
    TimeGrad & CRPS  & 0.606 & 1.212 & 0.647 & 0.775 & 0.397 & 0.826 & 1.140 & \underline{0.293} & 0.407 \\
    (2021) & QICE  & 6.731 & 9.488 & 6.693 & 6.962 & 7.118 & 9.464 & 6.519 & 7.378 & 4.581 \\
    \midrule
    CSDI  & CRPS  & 0.492 & 0.647 & 0.524 & 0.817 & 0.577 & 0.855 & 1.244 & 0.432 & 1.418 \\
    (2022) & QICE  & 3.107 & 5.331 & 2.828 & 8.106 & 7.506 & 7.864 & 7.693 & 9.957 & 13.613 \\
    \midrule
    TimeDiff & CRPS  & 0.465 & 0.471 & 0.464 & 0.316 & 0.750 & 0.433 & 1.153 & 0.700 & 0.771 \\
    (2023) & QICE  & 14.931 & 14.813 & 14.795 & 13.385 & 15.466 & 14.556 & 14.942 & 14.914 & 15.439 \\
    \midrule
    DiffusionTS & CRPS  & 0.603 & 1.168 & 0.574 & 1.035 & 0.633 & 1.251 & 1.612 & 0.470 & 0.668 \\
    (2024) & QICE  & 6.423 & 9.577 & 5.605 & 9.959 & 8.205 & 10.411 & 10.090 & \underline{6.627} & 5.958 \\
    \midrule
    TMDM  & CRPS  & 0.452 & 0.383 & 0.375 & 0.289 & 0.461 & 0.336 & 0.967 & 0.350 & 0.557 \\
    (2024) & QICE  & 2.821 & 4.471 & 2.567 & 2.610 & 10.562 & 6.393 & \underline{6.217} & 9.342 & 10.676 \\
    \midrule
    NsDiff & CRPS  & \underline{0.392} & \underline{0.358} & \underline{0.346} & \underline{0.256} & \underline{0.290} & \underline{0.324} & \textbf{0.806} & 0.300 & \underline{0.378} \\
    (2025) & QICE  & \textbf{1.470} & \textbf{2.074} & \textbf{2.041} & \underline{2.030} & \underline{6.685} & \underline{5.930} & \textbf{5.598} & 6.820 & \underline{3.601} \\
    \midrule
    TimeES    & CRPS  & \textbf{0.349} & \textbf{0.320} & \textbf{0.310} & \textbf{0.230} & \textbf{0.205} & \textbf{0.270} & \underline{0.834} & \textbf{0.193} & \textbf{0.245} \\
    (ours) & QICE  & \underline{1.921} & \underline{2.479} & \underline{2.247} & \textbf{1.984} & \textbf{1.941} & \textbf{4.302} & 6.930 & \textbf{4.768} & \textbf{2.177} \\
    \bottomrule
    \end{tabular}%
  \label{tab:probabilistic forecasting1}%
\end{table}%

\begin{table}[htbp]
  \centering
    \footnotesize
\caption{Deterministic forecasting performance (MSE, MAE) across nine datasets.}
\setlength{\tabcolsep}{6.3pt}
    \begin{tabular}{ccccccccccc}
    \toprule
    Models & Datasets & ETTh1 & ETTh2 & ETTm1 & ETTm2 & ECL   & EXG   & ILI   & Solar & Traffic \\
    \midrule
    TimeGrad & MAE   & 0.813 & 1.496 & 0.831 & 0.967 & 0.504 & 1.058 & 1.414 & 0.446 & 0.535 \\
    (2021) & MSE   & 1.062 & 3.462 & 1.218 & 1.690 & 0.505 & 1.567 & 4.197 & 0.475 & 0.983 \\
    \midrule
    CSDI  & MAE   & 0.708 & 0.900 & 0.752 & 1.069 & 0.822 & 1.081 & 1.481 & 0.675 & 0.925 \\
    (2022) & MSE   & 0.949 & 1.226 & 1.002 & 1.723 & 1.007 & 1.701 & 4.515 & 0.763 & 1.731 \\
    \midrule
    TimeDiff & MAE   & \underline{0.479} & \underline{0.485} & 0.477 & \underline{0.333} & 0.764 & 0.446 & 1.169 & 0.713 & 0.784 \\
    (2023) & MSE   & \underline{0.517} & \underline{0.456} & 0.537 & \underline{0.268} & 0.879 & 0.402 & 3.958 & 0.821 & 1.350 \\
    \midrule
    DiffusionTS & MAE   & 0.774 & 1.411 & 0.744 & 1.232 & 0.856 & 1.564 & 1.788 & 0.740 & 0.815 \\
    (2024) & MSE   & 1.089 & 3.273 & 1.030 & 2.372 & 1.072 & 3.628 & 6.053 & 0.749 & 1.473 \\
    \midrule
    TMDM  & MAE   & 0.607 & 0.490 & \underline{0.455} & 0.395 & 0.359 & 0.430 & 1.175 & 0.316 & 0.425 \\
    (2024) & MSE   & 0.696 & 0.512 & 0.494 & 0.315 & 0.257 & 0.334 & 3.636 & 0.250 & 0.679 \\
    \midrule
    NsDiff & MAE   & 0.523 & 0.490 & \underline{0.455} & 0.352 & \underline{0.306} & \underline{0.412} & \textbf{0.985} & \underline{0.307} & \underline{0.373} \\
    (2025) & MSE   & 0.594 & 0.514 & \underline{0.488} & 0.281 & \underline{0.209} & \underline{0.300} & \textbf{2.846} & \underline{0.242} & \underline{0.637} \\
    \midrule
    TimeES & MAE   & \textbf{0.455} & \textbf{0.408} & \textbf{0.403} & \textbf{0.292} & \textbf{0.263} & \textbf{0.347} & \underline{1.043} & \textbf{0.234} & \textbf{0.304} \\
    (ours) & MSE   & \textbf{0.474} & \textbf{0.376} & \textbf{0.371} & \textbf{0.235} & \textbf{0.180} & \textbf{0.242} & \underline{3.052} & \textbf{0.179} & \textbf{0.456} \\
    \bottomrule
    \end{tabular}%
  \label{tab:probabilistic forecasting2}%
\end{table}%

\label{apdx:uniandmulti_results_apdx_exp}
\begin{table*}[htbp]
  \centering
  \footnotesize
  \caption{Full results of the single variate forecasting task. We compare extensive competitive models under different prediction lengths following the setting of iTransformer~\cite{liu2023itransformer}. The input sequence length is set to 96 for all baselines. The statistical significance is within 1\%.}
  \setlength{\tabcolsep}{2.1pt}
\begin{tabular}{cc|cc|cc|cc|cc|cc|cc|cc|cc}
    \toprule
    \multicolumn{2}{c}{Model} & \multicolumn{2}{c}{Autoformer} & \multicolumn{2}{c}{FEDformer} & \multicolumn{2}{c}{TimesNet} & \multicolumn{2}{c}{FITS} & \multicolumn{2}{c}{FreTS} & \multicolumn{2}{c}{PatchTST} & \multicolumn{2}{c}{iTransformer} & \multicolumn{2}{c}{TimeES} \\
\cmidrule{3-18}    \multicolumn{2}{c}{Metrics} & MSE    & MAE    & MSE    & MAE    & MSE    & MAE    & MSE    & MAE    & MSE    & MAE    & MSE    & MAE    & MSE    & MAE    & MSE    & MAE \\
  \midrule
  \multicolumn{1}{c|}{\multirow{5}{*}{\rotatebox{90}{ETTm1}}} & 96    & \multicolumn{1}{r}{0.886 } & \multicolumn{1}{r|}{0.636 } & \multicolumn{1}{r}{0.713 } & \multicolumn{1}{r|}{0.598 } & \multicolumn{1}{r}{0.682 } & \multicolumn{1}{r|}{0.551 } & \multicolumn{1}{r}{0.643 } & \multicolumn{1}{r|}{0.511 } & \multicolumn{1}{r}{\textbf{0.600} } & \multicolumn{1}{r|}{\underline{0.508} } & \multicolumn{1}{r}{0.624 } & \multicolumn{1}{r|}{\underline{0.508} } & \multicolumn{1}{r}{0.622 } & \multicolumn{1}{r|}{0.517 } & \multicolumn{1}{r}{\underline{0.607} } & \multicolumn{1}{r}{\textbf{0.505} } \\
  \multicolumn{1}{c|}{} & 192   & \multicolumn{1}{r}{1.252 } & \multicolumn{1}{r|}{0.798 } & \multicolumn{1}{r}{0.770 } & \multicolumn{1}{r|}{0.614 } & \multicolumn{1}{r}{0.748 } & \multicolumn{1}{r|}{0.568 } & \multicolumn{1}{r}{0.745 } & \multicolumn{1}{r|}{\underline{0.545} } & \multicolumn{1}{r}{\underline{0.710} } & \multicolumn{1}{r|}{0.553 } & \multicolumn{1}{r}{\underline{0.710} } & \multicolumn{1}{r|}{\underline{0.545} } & \multicolumn{1}{r}{0.730 } & \multicolumn{1}{r|}{0.554 } & \multicolumn{1}{r}{\textbf{0.690} } & \multicolumn{1}{r}{\textbf{0.534} } \\
\bottomrule
    \end{tabular}%
  \label{tab:uni_results}%
\end{table*}%

\begin{table}[htbp]
  \centering
  \footnotesize
  \setlength{\tabcolsep}{2.1pt}
  \caption{Full results of the multivariate forecasting task. We compare extensive competitive models under different prediction lengths following the setting of iTransformer~\cite{liu2023itransformer}. The input sequence length is set to 96 for all baselines. The statistical significance is within 1\%.}
    \begin{tabular}{cc|cc|cc|cc|cc|cc|cc|cc|cc}
    \toprule
    \multicolumn{2}{c}{Model} & \multicolumn{2}{c}{Autoformer} & \multicolumn{2}{c}{FEDformer} & \multicolumn{2}{c}{TimesNet} & \multicolumn{2}{c}{FITS} & \multicolumn{2}{c}{FreTS} & \multicolumn{2}{c}{PatchTST} & \multicolumn{2}{c}{iTransformer} & \multicolumn{2}{c}{TimeES} \\
\cmidrule{3-18}    \multicolumn{2}{c}{Metrics} & MSE    & MAE    & MSE    & MAE    & MSE    & MAE    & MSE    & MAE    & MSE    & MAE    & MSE    & MAE    & MSE    & MAE    & MSE    & MAE \\
    \midrule
    \multicolumn{1}{c|}{\multirow{5}{*}{\rotatebox{90}{ETTm1}}} & 96    & \multicolumn{1}{r}{0.475 } & \multicolumn{1}{r|}{0.505 } & \multicolumn{1}{r}{0.379 } & \multicolumn{1}{r|}{0.419 } & \multicolumn{1}{r}{0.338 } & \multicolumn{1}{r|}{0.375 } & \multicolumn{1}{r}{0.353 } & \multicolumn{1}{r|}{0.378 } & \multicolumn{1}{r}{0.352 } & \multicolumn{1}{r|}{0.385 } & \multicolumn{1}{r}{\underline{0.329} } & \multicolumn{1}{r|}{\underline{0.367} } & \multicolumn{1}{r}{0.334 } & \multicolumn{1}{r|}{0.368 } & \multicolumn{1}{r}{\textbf{0.328} } & \multicolumn{1}{r}{\textbf{0.366} } \\
    \multicolumn{1}{c|}{} & 192   & \multicolumn{1}{r}{0.553 } & \multicolumn{1}{r|}{0.496 } & \multicolumn{1}{r}{0.426 } & \multicolumn{1}{r|}{0.441 } & \multicolumn{1}{r}{0.374 } & \multicolumn{1}{r|}{0.387 } & \multicolumn{1}{r}{0.389 } & \multicolumn{1}{r|}{0.400 } & \multicolumn{1}{r}{0.390 } & \multicolumn{1}{r|}{0.407 } & \multicolumn{1}{r}{\underline{0.367} } & \multicolumn{1}{r|}{\underline{0.385} } & \multicolumn{1}{r}{0.377 } & \multicolumn{1}{r|}{0.391 } & \multicolumn{1}{r}{\textbf{0.364} } & \multicolumn{1}{r}{\textbf{0.384} } \\
\bottomrule
    \end{tabular}%
  \label{tab:multi_results}%
\end{table}%

\begin{table}[htbp]
  \centering
  \caption{We evaluate TimeES as a decomposition framework on long-term single variate and multivariate forecasting tasks in ETTs datasets, the best results are denote in \textbf{bold} characters. }
  \footnotesize
  \setlength{\tabcolsep}{3.6pt}
    \begin{tabular}{c|c|cccc|cc|c|ccc|cc}
    \toprule
    \multicolumn{14}{c}{\textbf{Single-Variates}} \\
    \midrule
    \multicolumn{2}{c}{Model} & \multicolumn{2}{c}{PatchTST} & \multicolumn{2}{c}{+TimeES} & \multicolumn{2}{c|}{Improvements} & \multicolumn{2}{c}{iTransformer} & \multicolumn{2}{c}{+TimeES} & \multicolumn{2}{c}{Improvements} \\
\cmidrule{3-14}    \multicolumn{2}{c}{Metrics} & MSE   & MAE   & MSE   & MAE   & \multicolumn{1}{c}{MSE} & \multicolumn{1}{c|}{MAE} & \multicolumn{1}{c}{MSE} & MAE   & MSE   & MAE   & \multicolumn{1}{c}{MSE} & \multicolumn{1}{c}{MAE} \\
    \midrule
    \multirow{5}[4]{*}{ETTm1} & 96    & 0.624  & 0.508  & \textbf{0.603 } & \textbf{0.505 } & 3.37\% & 0.63\% & \multicolumn{1}{c}{0.622 } & 0.517  & \textbf{0.612 } & \textbf{0.496 } & 1.56\% & 4.02\% \\
          & 192   & 0.710  & 0.545  & \textbf{0.704 } & \textbf{0.533 } & 0.88\% & 2.35\% & \multicolumn{1}{c}{0.730 } & 0.554  & \textbf{0.721 } & \textbf{0.544 } & 1.35\% & 1.77\% \\
          & 336   & 0.786  & 0.585  & \textbf{0.780 } & \textbf{0.574 } & 0.70\% & 1.80\% & \multicolumn{1}{c}{0.829 } & 0.609  & \textbf{0.787 } & \textbf{0.576 } & 5.07\% & 5.39\% \\
          & 720   & 0.931  & 0.646  & \textbf{0.921 } & \textbf{0.635 } & 1.15\% & 1.63\% & \multicolumn{1}{c}{0.931 } & 0.642  & \textbf{0.923 } & \textbf{0.634 } & 0.86\% & 1.23\% \\
\cmidrule{3-14}          & Avg   & 0.763  & 0.571  & \textbf{0.752 } & \textbf{0.562 } & 1.42\% & 1.62\% & \multicolumn{1}{c}{0.778 } & 0.580  & \textbf{0.760 } & \textbf{0.563 } & 2.24\% & 3.07\% \\
\bottomrule
    \end{tabular}%
  \label{tab:decompo_results}%
\end{table}%

\clearpage
\section{Additional Showcases}
\label{apdx:other_prob_showcases}

We present other probabilistic forecasting showcases in Figure~\ref{fig:other_showcases_NES}.

\begin{figure}[!h]
\vskip -0.05in
\begin{center}
\centerline{\includegraphics[width=1\columnwidth]{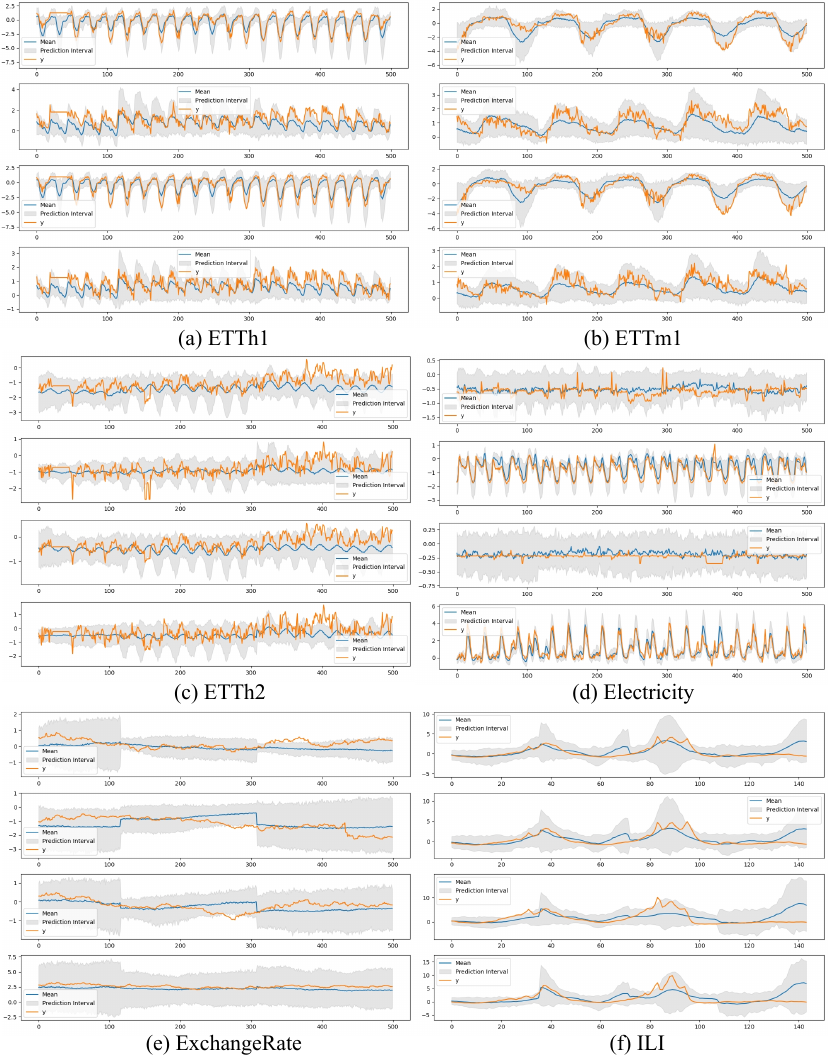}}
\vskip -0.1in 
\caption{(a-f) Predictive uncertainty quantification by TimeES across different datasets. We plot the last 500 predicted time steps of the first four variables in each dataset. The orange line shows the ground truth, the blue line represents the predicted mean, and the gray shaded region indicates the 96\% prediction interval.}
\label{fig:other_showcases_NES}
\end{center}
\vskip -0.3in 
\end{figure}

\section{Additional Interpretability Analysis}

We present additional results of the spectral decomposition analysis in Figure~\ref{fig:interpretability_analysis_adpx_fig}.

\label{apdx:interpretability_analysis}
\begin{figure}[!h]
\vskip -0.05in
\begin{center}
\centerline{\includegraphics[width=0.95\columnwidth]{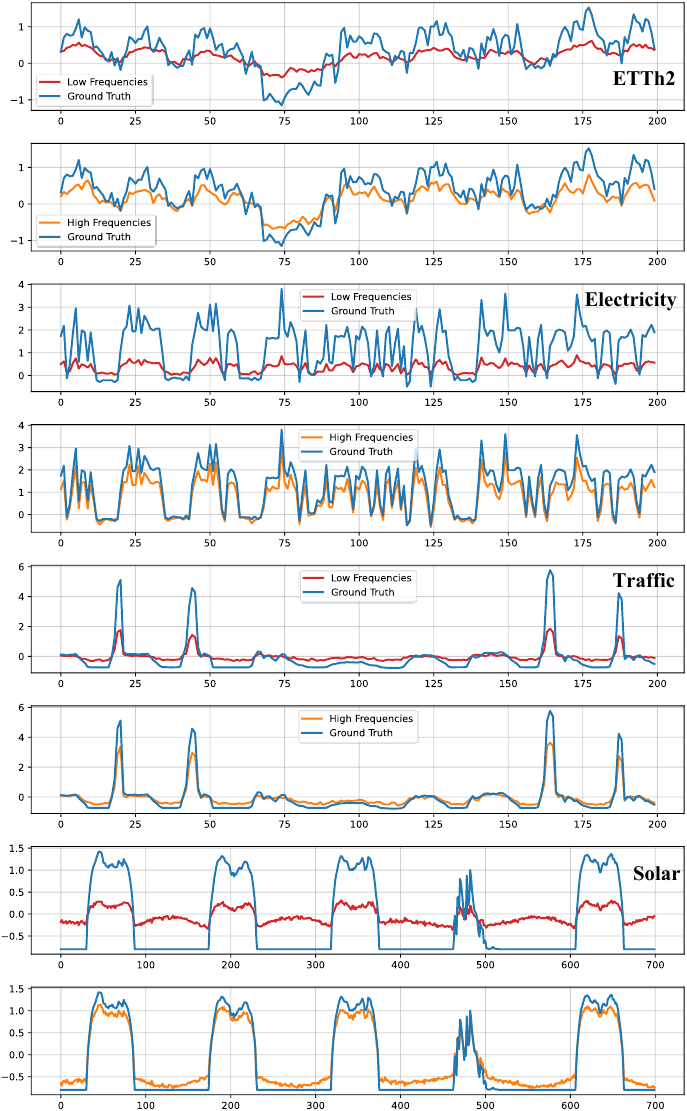}}
\caption{Spectral decomposition analysis on four dataset. We visualize the high and low frequencies.}
\label{fig:interpretability_analysis_adpx_fig}
\end{center}
\end{figure}

\clearpage

\textbf{Per-channel uncertainty attribution.} The variance decomposition in Equation~\eqref{eq:variance_decomposition} attributes the predictive uncertainty of each channel to individual frequencies. Since the learned $\sigma_{W_k}$ are nearly constant (Appendix~\ref{apdx:spectral_rv}), the structure of the uncertainty lies in the amplitude, and the quantity to read is the share of each frequency in the predictive variance, $dS(n,\omega_k)/\sum_{k'}dS(n,\omega_{k'})$. We compute these shares on ETTh1 ($M=168$, i.e., 85 frequencies, for which a uniform share would be 1.18\%), averaged over all forecast steps of 672 test windows. Table~\ref{tab:per_channel_uncertainty} reports the shares at the trend (DC), weekly, daily, 12-hour, and 8-hour frequencies, and Figure~\ref{fig:per_channel_uncertainty} shows all 85 frequencies. For the load channels, the trend carries 1.4 to 2.2 times the uniform share, with the largest values for HULL and MULL, whereas OT (oil temperature) places more uncertainty in the weekly band than any load channel (1.51\% vs.\ 0.44\% to 1.05\%), consistent with oil temperature being a slowly varying thermal state.

\begin{table}[htbp]
\centering
\footnotesize
\caption{Share (\%) of the predictive variance contributed by selected frequencies on ETTh1, averaged over 672 test windows. A uniform share would be 1.18\%. The largest value in each column is in \textbf{bold}.}
\label{tab:per_channel_uncertainty}
\begin{tabular}{lccccc}
\toprule
Channel & Trend (DC) & Weekly (168h) & Daily (24h) & 12h & 8h \\
\midrule
HUFL (high useful load) & 1.95 & 0.46 & 1.60 & 0.80 & 0.95 \\
HULL (high useless load) & 2.55 & 0.83 & 1.34 & 0.71 & 0.78 \\
MUFL (middle useful load) & 2.02 & 0.44 & 1.57 & 0.79 & 0.94 \\
MULL (middle useless load) & \textbf{2.56} & 0.84 & 1.32 & 0.70 & 0.81 \\
LUFL (low useful load) & 1.63 & 0.84 & \textbf{1.81} & \textbf{0.94} & 0.96 \\
LULL (low useless load) & 2.07 & 1.05 & 1.34 & 0.83 & 1.10 \\
OT (oil temperature) & 1.73 & \textbf{1.51} & 1.54 & 0.69 & \textbf{1.14} \\
\bottomrule
\end{tabular}
\end{table}

\begin{figure}[htbp]
\centering
\includegraphics[width=\linewidth]{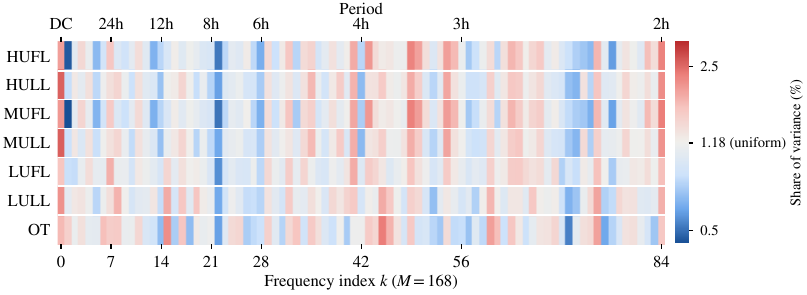}
\caption{Share of the predictive variance contributed by each frequency on ETTh1 for each channel ($M=168$, 85 frequencies). Red and blue cells indicate shares above and below the uniform share of 1.18\%, respectively; $k=1$ corresponds to the weekly period.}
\label{fig:per_channel_uncertainty}
\end{figure}

\section{Relation to Koopman Operators and Dynamic Mode Decomposition}
\label{apdx:koopman_dmd}
Koopman operator theory represents nonlinear dynamics by a linear operator acting on a lifted space of observables. Besides deterministic forecasters such as KNF~\cite{wang2022koopman}, Koopa~\cite{liu2024koopa}, and SKOLR~\cite{zhang2025skolr}, which we compare with in Appendix~\ref{subsec:addtional_baselines}, this line of work also learns continuous spectra~\cite{lusch2018deep} and includes probabilistic formulations such as Deep Probabilistic Koopman~\cite{mallen2021deep}, Koopman VAE~\cite{naiman2024generative}, and KooNPro~\cite{zheng2025koonpro}. Both Koopman models and TimeES describe time series through a spectrum, but they treat the spectrum differently. Taking KooNPro as a representative probabilistic Koopman model, we compare the two in terms of the randomness in the spectrum, the spectral decomposition, and the parameterization of non-stationarity, and then discuss the forecasting procedure, the relation to DMD, and an empirical comparison.

\subsection{Randomness in the Spectrum}
Both models place a Gaussian distribution on a spectral quantity, but on different ones. KooNPro places it on the frequency, $\omega_t \sim \mathcal{N}(\mu(\kappa_t), \sigma(\kappa_t))$, so its variance is a variance \emph{of frequencies} and expresses uncertainty about which dynamics generated the data. In evolutionary spectra, the spectrum \emph{is} the variance of the observation, resolved by frequency. Since the $W_k$ in Theorem~\ref{theorem:main_theorem} are independent, the predictive variance decomposes additively over the retained frequencies:
\begin{equation}
dS(n,\omega_k) = \frac{1}{M}\,|A(n,\omega_k)|^2\,\sigma_{W_k}^2, \qquad \sum_{k\in\mathcal{K}} dS(n,\omega_k) = \operatorname{Var}[X_n],
\label{eq:variance_decomposition}
\end{equation}
where $\sigma_{W_k}^2 = \mathbb{E}|W_k-\mu_{W_k}|^2$.\footnote{For real-valued series synthesized under Hermitian symmetry (Theorem~\ref{theorem:hermitian}), each retained frequency is paired with its conjugate, which changes the constant $1/M$ but not the additive decomposition. The check below uses the exact constant of our implementation.} We verify this identity with the trained ETTh1 probabilistic model on one test batch (32 windows, 7 channels, and 192 steps), comparing the analytic sum with the empirical variance of 4{,}000 sampled paths. Averaged over all steps, the analytic sum is 1.310 and the empirical variance is 1.307, with a mean relative error of 1.8\% (maximum 9.3\%), which matches the sampling error expected from 4{,}000 Gaussian samples (a mean absolute relative error of $\sqrt{4/(4000\pi)}\approx 1.8\%$). The randomness modeled by TimeES is therefore the stochastic excitation of the process itself, i.e., the discrete counterpart of $\mathbb{E}|dZ(\omega)|^2$ in Definition~\ref{theorem:evolutionary_spectra}.

\subsection{Spectral Decomposition}
The two spectra decompose different objects. A Koopman spectrum decomposes the \emph{evolution operator} in a lifted space: the modes live in the latent state $h$ and reach the observation only through a decoder $\phi^{-1}$, where $\phi$ denotes the encoder. Since $\phi^{-1}$ is nonlinear, $\phi^{-1}\big(\sum_j h^{(j)}\big) \neq \sum_j \phi^{-1}\big(h^{(j)}\big)$, so individual modes are not separable signal components. An evolutionary spectrum decomposes the \emph{observation itself}, and it does so additively:
\begin{equation}
x_n = \sum_{k\in\mathcal{K}} x_n^{(k)}, \qquad x_n^{(k)} = \frac{1}{\sqrt{M}}\,A(n,\omega_k)\,W_k\,e^{i\omega_k n}.
\end{equation}
Any sub-band of frequencies therefore yields a valid time-domain component on its own, without a decoder. This additivity underlies the band-wise reconstruction in Section~\ref{subsec:interpretability_analysis} and the per-channel uncertainty attribution in Appendix~\ref{apdx:interpretability_analysis}.

\subsection{Parameterization of Non-stationarity}
\label{apdx:subsec:am_fm}
KooNPro moves the frequencies ($\omega_t$ is re-drawn at every step), which is a form of frequency modulation. DES fixes the frequencies on the DFT grid of the observation and lets their amplitudes evolve, which is amplitude modulation, the form prescribed by Priestley's construction~\cite{priestley1965evolutionary}. This costs no generality: since $A(n,\omega_k)$ is an unconstrained function of $n$, an off-grid component $c\,e^{i\omega^\ast n}$ is represented exactly on a grid frequency $\omega_k$ by $A(n,\omega_k)=c\,e^{i(\omega^\ast-\omega_k)n}$, so a time-varying amplitude already subsumes a shifted frequency. Fixing the grid also keeps the synthesis basis orthogonal, without which $|A|^2$ would no longer be a power in the Parseval sense.

\begin{table}[htbp]
\centering
\footnotesize
\caption{CRPS (mean$\pm$std over 3 seeds) of TimeES with frequencies fixed on the DFT grid or learned jointly with the amplitudes.}
\label{tab:learnable_freq}
\begin{tabular}{lcc}
\toprule
Dataset & Fixed grid & Learnable \\
\midrule
ETTh1 & \textbf{0.363$\pm$0.006} & 0.372$\pm$0.008 \\
Off-grid synthetic & \textbf{0.156$\pm$0.018} & 0.157$\pm$0.017 \\
\bottomrule
\end{tabular}
\end{table}

\textbf{Fixed grid vs.\ learnable frequencies.} To test whether the fixed grid restricts the model, we make the frequencies learnable (initialized at the selected bins and trained jointly with the amplitudes) and compare against the fixed grid on ETTh1 and on a synthetic series whose periods deliberately miss the grid. The synthetic series is a sum of two sinusoids with periods of 30 and 45 steps plus Gaussian noise; with $M=168$, these correspond to 5.60 and 3.73 cycles per window, both between DFT bins. Both variants use the same training settings and seeds. As shown in Table~\ref{tab:learnable_freq}, the accuracy is unchanged and the frequencies stay on the grid: the largest movement of any frequency is 0.065 of a bin (0.02 to 0.03 on average). On the off-grid series, reaching the true components would require the nearest selected bins, 6 and 4, to move by 0.40 and 0.27 of a bin, roughly six and four times further than the optimizer actually moves them. The freedom to leave the grid is therefore available but unused, as the above argument predicts.

\subsection{Recursive vs.\ Closed-form Forecasting}
KooNPro draws a global dynamics representation $S$ from the neural-process context, which is frozen as $S_C$ at test time, and then applies its Koopman operator $\mathcal{U}$ recursively. DES instead predicts the whole amplitude field at once, so the time index enters only through a fixed Fourier phase and the entire horizon is a single matrix product:
\begin{equation}
z_H=\phi^{-1}\big(\mathcal{U}^{H-1}(\phi(z_1), S_C)\big) \qquad \text{vs.} \qquad \mathbf{y}=\frac{1}{\sqrt{M}}\big(\hat{\mathbf{A}}\odot\mathbf{F}\big)\mathbf{w},
\end{equation}
where $H$ denotes the forecast horizon. The left side requires $H-1$ sequential applications of $\mathcal{U}$, and since the frequency sampled at each step feeds into the next one, perturbations propagate rather than average out. The right side involves no recursion: lengthening the horizon adds terms to a sum rather than links to a chain, and a single draw of $\{W_k\}$ keeps each sampled path spectrally coherent.

\subsection{Relation to Dynamic Mode Decomposition}
DMD~\cite{schmid2022dynamic} fits a linear operator $\mathbf{G}$ between consecutive snapshot matrices, $\mathbf{X}_1\approx\mathbf{G}\mathbf{X}_0$, and reconstructs the series from its eigen-decomposition as $\mathbf{x}_n=\sum_j b_j\boldsymbol{\psi}_j e^{(\gamma_j+i\nu_j)n}$, where $\boldsymbol{\psi}_j$ is a mode with amplitude $b_j$, frequency $\nu_j$, and growth rate $\gamma_j$. Each DMD mode thus oscillates at a single frequency under an exponential envelope, and both are fixed once the operator is fitted. In the notation of DES, a DMD reconstruction corresponds to deterministic weights $W_k=1$ and amplitudes of the restricted form $A(n,\omega_k)\propto e^{\gamma_j n}e^{i(\nu_j-\omega_k)n}$, following the off-grid argument in Appendix~\ref{apdx:subsec:am_fm}. DES removes this restriction: the amplitudes are arbitrary functions of $n$ predicted from the lookback window, the whole horizon is produced at once instead of by iterating a one-step operator, and the random weights $W_k$ turn the decomposition into a generative model.

\subsection{Empirical Comparison with KooNPro}
The experimental protocols of KooNPro and TimeES differ in both history length and horizon: KooNPro uses a history of $L=10$ steps on the ETT datasets and a horizon of $H=24$, whereas our benchmark uses $L=168$ and $H=192$. We therefore ran the official implementation of KooNPro under both history lengths and at both horizons, with 3 seeds for each setting, and report the CRPS in Table~\ref{tab:koonpro}. TimeES outperforms KooNPro in all settings. At $H=24$, it reduces the CRPS by 27\% to 54\% with $L=168$, and by 33\% to 42\% with $L=10$, except on ETTm1, where 10 steps of 15-minute data cover only 2.5 hours and the margin narrows to 5\%. At the horizon of our benchmark ($H=192$), the CRPS of KooNPro increases by roughly an order of magnitude, and 83\% to 88\% of the test points fall below its 10th predictive percentile, which is consistent with the error accumulation of recursive forecasting discussed above.

\begin{table}[htbp]
\centering
\footnotesize
\caption{CRPS comparison with KooNPro on the ETT datasets (3 seeds per setting). $L$ denotes the history length and $H$ the forecast horizon; the original setting of KooNPro is $L=10$ and $H=24$, and the TimeES results for $L=168$, $H=192$ are from Table~\ref{tab:probabilistic_forecasting_results}. \textbf{Bold} indicates the better result.}
\label{tab:koonpro}
\begin{tabular}{llcccc}
\toprule
Setting & Model & ETTh1 & ETTh2 & ETTm1 & ETTm2 \\
\midrule
\multirow{2}{*}{$L=10$, $H=24$} & KooNPro & 0.527 & 0.351 & 0.395 & 0.277 \\
 & TimeES & \textbf{0.355} & \textbf{0.202} & \textbf{0.375} & \textbf{0.175} \\
\midrule
\multirow{2}{*}{$L=168$, $H=24$} & KooNPro & 0.398 & 0.380 & 0.352 & 0.290 \\
 & TimeES & \textbf{0.292} & \textbf{0.206} & \textbf{0.175} & \textbf{0.132} \\
\midrule
\multirow{2}{*}{$L=168$, $H=192$} & KooNPro & 3.079 & 2.922 & 3.465 & 3.656 \\
 & TimeES & \textbf{0.349} & \textbf{0.320} & \textbf{0.310} & \textbf{0.230} \\
\bottomrule
\end{tabular}
\end{table}

\section{Further Discussion and Analysis}
\label{apdx:further_discussion}

\subsection{Frequency Support and Evolving Spectra}
\label{apdx:freq_support}
The energy-based selection in Section~\ref{Sec:4.2} fixes a set of frequencies $\mathcal{K}$, whereas the premise of TimeES is that spectra evolve. The two are compatible because they concern different objects. In DES (Theorem~\ref{theorem:main_theorem}), non-stationarity is carried by the time-varying complex amplitude $A(n,\omega_k)$, not by the support $\mathcal{K}$ of the retained frequencies. $\mathcal{K}$ only determines which columns of the complete DFT grid are modeled, and a time-varying amplitude on a fixed grid can represent evolving spectral content, including frequency drift, since amplitude modulation between neighboring grid frequencies is equivalent to a varying instantaneous frequency (Appendix~\ref{apdx:subsec:am_fm}). The chirp experiment in Appendix~\ref{apdx:chirp_synthetic_exp} illustrates this: the test region contains frequencies never seen during training, yet TimeES tracks both the evolving frequency and the evolving amplitude. Moreover, $\mathcal{K}$ is estimated from the average periodogram of the entire training split rather than from individual lookback windows, so it reflects the spectral support of the dataset rather than a snapshot of a single window.

To quantify how well this support transfers to unseen data, we select $\mathcal{K}$ on the training split only ($M=96$ and $r=0.9$) and measure the fraction of the spectral energy of the test split captured by these frequencies. As shown in Table~\ref{tab:energy_transfer}, the train-selected support captures 92.3\% of the test energy on average, close to the target ratio $r=0.9$, while the evolution of the spectrum within this support is modeled by $A(n,\omega_k)$.

\begin{table}[htbp]
\centering
\footnotesize
\setlength{\tabcolsep}{4pt}
\caption{Fraction of the spectral energy of the test split captured by the frequencies selected on the training split ($M=96$, $r=0.9$).}
\label{tab:energy_transfer}
\begin{tabular}{lccccccccccc}
\toprule
Dataset & ETTh1 & ETTh2 & ETTm1 & ETTm2 & ECL & EXG & Traffic & Weather & Solar & ILI & Avg. \\
\midrule
Coverage & 0.885 & 0.862 & 0.914 & 0.900 & 0.915 & 0.992 & 0.904 & 0.965 & 0.928 & 0.963 & 0.923 \\
\bottomrule
\end{tabular}
\end{table}

\subsection{The Spectral Random Variables}
\label{apdx:spectral_rv}
\textbf{Parameterization.} The parameters $\mu_{W_k}$ and $\sigma_{W_k}$ are learnable constants, one pair for each selected frequency and each channel ($K\times C$ pairs), shared across the forecast horizon. The horizon dependence of the predictive distribution therefore enters entirely through $A(n,\omega_k)$, consistent with the ES construction, in which the randomness is carried by the spectral measure $dZ(\omega)$ with $\mathbb{E}|dZ(\omega)|^2=d\omega$ and all non-stationary structure is carried by $A(t,\omega)$. Training recovers this structure: on ETTh1, the learned $\sigma_{W_k}$ stay within about 4\% of their initial value $\sqrt{2}$ (unit variance for both the real and imaginary parts), ranging from 1.36 to 1.42 over all 595 frequency-channel pairs. The spectral variance of $\mathbf{w}$ is thus nearly flat across frequencies, and the modulation of uncertainty over time and frequency is expressed by $|A(n,\omega_k)|$, with the per-step predictive variance given by Equation~\eqref{eq:variance_decomposition}.

\textbf{Expressiveness.} The synthesis is linear in $\mathbf{w}$, but its coefficients are not: $A(n,\omega_k)$ is a nonlinear function of the input, so DES is a conditional location-scale model whose mean and full predictive covariance are both predicted nonlinearly from the history. This yields two properties that models with independent per-step noise lack. First, a single draw of $\{W_k\}$ is shared by the entire horizon, so a sampled trajectory stays spectrally coherent instead of fluctuating independently at each step. Second, the covariance between any two forecast steps is available in closed form rather than only through sampling; for the complex synthesis in Theorem~\ref{theorem:main_theorem}, $\operatorname{Cov}[X_n,X_m]=\frac{1}{M}\sum_{k}A(n,\omega_k)\overline{A(m,\omega_k)}\,\sigma_{W_k}^2 e^{i\omega_k(n-m)}$. Empirically, the Gaussian assumption does not appear restrictive: TimeES attains the best QICE on 5 of 9 datasets and the second best on 3 (Table~\ref{tab:probabilistic forecasting1}), against diffusion-based baselines with unconstrained generators. Its cost is limited flexibility in the tails of the per-step distribution, and extending DES with non-Gaussian or mixture spectral measures is a natural direction for future work.

\subsection{Sensitivity to Input Length}
\label{apdx:input_length}
We examine the effect of the input length $N$ on probabilistic forecasting, setting $M=N$ and keeping all other settings of Section~\ref{subsec:probabilistic_forecasting} (horizon 192, 3 seeds). As shown in Table~\ref{tab:input_length}, TimeES is relatively robust to the input length. On ETTh1, even a 24-step input (one day) stays within 2\% of the default setting. On ETTm1 (15-minute granularity), only the 24-step input degrades markedly (+28\%), since 6 hours do not cover a daily cycle; from 72 steps on, the CRPS stays within 4\% of the default setting, and the 96- and 336-step inputs improve on it.

\begin{table}[htbp]
\centering
\footnotesize
\caption{CRPS of TimeES with different input lengths (horizon 192, mean over 3 seeds). The default input length is 168. \textbf{Bold} indicates the best result.}
\label{tab:input_length}
\begin{tabular}{lccccc}
\toprule
Input length & 24 & 72 & 96 & 168 (default) & 336 \\
\midrule
ETTh1 & 0.356 & \textbf{0.336} & 0.339 & 0.349 & 0.369 \\
ETTm1 & 0.398 & 0.321 & \textbf{0.298} & 0.310 & 0.299 \\
\bottomrule
\end{tabular}
\end{table}

\subsection{Operating Range under Limited Data}
\label{apdx:limited_data}
TimeES estimates the evolutionary spectrum directly from data rather than assuming an analytic form, which underlies its gains in Table~\ref{tab:probabilistic_forecasting_results}: it achieves the best CRPS on 8 of 9 datasets and the best QICE on 5 of 9 (Table~\ref{tab:probabilistic forecasting1}). The same property also sets its boundary, since estimating a full time-frequency amplitude field requires sufficient samples. ILI is the extreme case in our benchmark: its 966 time steps yield only 473 training windows, compared with 17{,}420 time steps for ETTh1, and it is also where the variance across seeds is largest (a CRPS standard deviation of 0.172 over 3 seeds, compared with less than 0.01 on ETTh1). NsDiff~\cite{ye2025non} is less affected in this regime because its time-varying variance is estimated from a trailing sliding window over the observed series, which remains well defined with few samples and adapts freely to asymmetric residuals; the same flexibility explains its better decile-level QICE on ETTh1, ETTh2, and ETTm1. Learning evolutionary spectra under limited data is therefore an important direction for future work.

\clearpage

\newpage
\section*{NeurIPS Paper Checklist}

\begin{enumerate}

\item {\bf Claims}
    \item[] Question: Do the main claims made in the abstract and introduction accurately reflect the paper's contributions and scope?
    \item[] Answer: \answerYes{} 
    \item[] Justification: The main claims made in the abstract and introduction accurately reflect the paper's contributions and scope.

\item {\bf Limitations}
    \item[] Question: Does the paper discuss the limitations of the work performed by the authors?
    \item[] Answer: \answerYes{} 
    \item[] Justification: We discuss the limitations of the work in Section~\ref{sec:conclusion}.

\item {\bf Theory assumptions and proofs}
    \item[] Question: For each theoretical result, does the paper provide the full set of assumptions and a complete (and correct) proof?
    \item[] Answer: \answerYes{} 
    \item[] Justification: The assumptions and proofs are provided at Appendix~\ref{sec:proofs}.

    \item {\bf Experimental result reproducibility}
    \item[] Question: Does the paper fully disclose all the information needed to reproduce the main experimental results of the paper to the extent that it affects the main claims and/or conclusions of the paper (regardless of whether the code and data are provided or not)?
    \item[] Answer: \answerYes{} 
    \item[] Justification: The paper fully discloses all the information needed to reproduce the main experimental results at Section~\ref{sec:experiments}.

\item {\bf Open access to data and code}
    \item[] Question: Does the paper provide open access to the data and code, with sufficient instructions to faithfully reproduce the main experimental results, as described in supplemental material?
    \item[] Answer: \answerYes{} 
    \item[] Justification: We provide the link to code in the abstract.

\item {\bf Experimental setting/details}
    \item[] Question: Does the paper specify all the training and test details (e.g., data splits, hyperparameters, how they were chosen, type of optimizer) necessary to understand the results?
    \item[] Answer: \answerYes{}{} 
    \item[] Justification: Experimental settings/details are discussed in Section~\ref{sec:experiments} including all the information to reproduce the results.

\item {\bf Experiment statistical significance}
    \item[] Question: Does the paper report error bars suitably and correctly defined or other appropriate information about the statistical significance of the experiments?
    \item[] Answer: \answerYes{} 
    \item[] Justification: We state in Appendix~\ref{apdx:full_other_results} about statistical significance. The statistical significance of probabilistic and deterministic is under 3\%/1\%, respectively.

\item {\bf Experiments compute resources}
    \item[] Question: For each experiment, does the paper provide sufficient information on the computer resources (type of compute workers, memory, time of execution) needed to reproduce the experiments?
    \item[] Answer: \answerYes{}{} 
    \item[] Justification: The results are run in NVIDIA A6000 GTX GPU. Section~\ref{subsect:efficiency_analysis} shows the computational cost of various methods.
    
\item {\bf Code of ethics}
    \item[] Question: Does the research conducted in the paper conform, in every respect, with the NeurIPS Code of Ethics \url{https://neurips.cc/public/EthicsGuidelines}?
    \item[] Answer: \answerYes{} 
    \item[] Justification: The work conforms with the NeurIPS Code of Ethics.

\item {\bf Broader impacts}
    \item[] Question: Does the paper discuss both potential positive societal impacts and negative societal impacts of the work performed?
    \item[] Answer: \answerYes{} 
    \item[] Justification: The paper discusses both potential positive societal impacts and negative societal impacts of the work performed.
    
\item {\bf Safeguards}
    \item[] Question: Does the paper describe safeguards that have been put in place for responsible release of data or models that have a high risk for misuse (e.g., pre-trained language models, image generators, or scraped datasets)?
    \item[] Answer: \answerNA{} 
    \item[] Justification: The paper has no such risks.

\item {\bf Licenses for existing assets}
    \item[] Question: Are the creators or original owners of assets (e.g., code, data, models), used in the paper, properly credited and are the license and terms of use explicitly mentioned and properly respected?
    \item[] Answer: \answerYes{} 
    \item[] Justification: The creators or original owners of assets (e.g., code, data, models), used in the paper, properly credited and are the license and terms of use explicitly mentioned and properly respected.

\item {\bf New assets}
    \item[] Question: Are new assets introduced in the paper well documented and is the documentation provided alongside the assets?
    \item[] Answer: \answerYes{} 
    \item[] Justification:  We provide the model code in the abstract.

\item {\bf Crowdsourcing and research with human subjects}
    \item[] Question: For crowdsourcing experiments and research with human subjects, does the paper include the full text of instructions given to participants and screenshots, if applicable, as well as details about compensation (if any)? 
    \item[] Answer: \answerNA{} 
    \item[] Justification: the paper does not involve crowd-sourcing nor research with human subjects.

\item {\bf Institutional review board (IRB) approvals or equivalent for research with human subjects}
    \item[] Question: Does the paper describe potential risks incurred by study participants, whether such risks were disclosed to the subjects, and whether Institutional Review Board (IRB) approvals (or an equivalent approval/review based on the requirements of your country or institution) were obtained?
    \item[] Answer: \answerNA{} 
    \item[] Justification: The paper does not involve research with human subjects.

\item {\bf Declaration of LLM usage}
    \item[] Question: Does the paper describe the usage of LLMs if it is an important, original, or non-standard component of the core methods in this research? Note that if the LLM is used only for writing, editing, or formatting purposes and does \emph{not} impact the core methodology, scientific rigor, or originality of the research, declaration is not required.
    \item[] Answer: \answerNA{} 
    \item[] Justification: The core method development in this research does not involve LLMs as any important, original, or non-standard components.

\end{enumerate}

\end{document}